\PassOptionsToPackage{table}{xcolor}
\documentclass[sigconf,nonacm,screen,balance=false]{acmart}

\usepackage{booktabs}
\usepackage{multirow}
\usepackage{enumitem}
\usepackage{mathtools}
\usepackage{algorithm}
\usepackage{algorithmic}
\usepackage[capitalize,noabbrev]{cleveref}

\theoremstyle{acmdefinition}
\newtheorem{assumption}{Assumption}

\newcommand{\vx}{\mathbf{x}}

\newcommand{\cC}{\mathcal{C}}      %
\newcommand{\E}{\mathbb{E}}

\newcommand{\Regret}{\mathrm{Regret}}
\newcommand{\Div}{\mathrm{D}}
\DeclareMathOperator*{\argmin}{arg\,min}

\crefname{assumption}{Asm.}{Asms.}\Crefname{assumption}{Asm.}{Asms.}
\crefname{proposition}{Prop.}{Props.}\Crefname{proposition}{Prop.}{Props.}
\crefname{theorem}{Thm.}{Thms.}\Crefname{theorem}{Thm.}{Thms.}
\crefname{lemma}{Lem.}{Lems.}\Crefname{lemma}{Lem.}{Lems.}
\crefname{corollary}{Cor.}{Cors.}\Crefname{corollary}{Cor.}{Cors.}
\crefname{appendix}{App.}{Apps.}\Crefname{appendix}{App.}{Apps.}
\crefname{section}{Sec.}{Secs.}\Crefname{section}{Sec.}{Secs.}
\crefname{subsection}{Sec.}{Secs.}\Crefname{subsection}{Sec.}{Secs.}
\crefname{figure}{Fig.}{Figs.}\Crefname{figure}{Fig.}{Figs.}
\crefname{table}{Tab.}{Tabs.}\Crefname{table}{Tab.}{Tabs.}
\crefname{subappendix}{App.}{Apps.}\Crefname{subappendix}{App.}{Apps.}
\crefname{subsubsection}{Sec.}{Secs.}\Crefname{subsubsection}{Sec.}{Secs.}
\crefname{algorithm}{Alg.}{Algs.}\Crefname{algorithm}{Alg.}{Algs.}
\crefname{equation}{Eq.}{Eqs.}\Crefname{equation}{Eq.}{Eqs.}
\newcommand{\method}{\textsc{MICE}}

\definecolor{rankA}{HTML}{F5E2A8}\definecolor{rankB}{HTML}{F9EDC9}\definecolor{rankC}{HTML}{FCF6E6}
\newcommand{\first}[1]{\cellcolor{rankA}\textbf{#1}}
\newcommand{\second}[1]{\cellcolor{rankB}\underline{#1}}
\newcommand{\third}[1]{\cellcolor{rankC}\textit{#1}}

\setcopyright{none}
\renewcommand\footnotetextcopyrightpermission[1]{}
\acmISBN{}\acmDOI{}

\begin{document}

\title[Recall or Relearn?]{Recall or Relearn? A Mixture of In-Context Experts for Recurring Concept Drift}

\author{Zhangkai Wu}
\affiliation{\institution{University of Sydney}\city{Sydney}\country{Australia}}
\email{zhangkai.wu@sydney.edu.au}
\author{Xu Zhang}
\affiliation{\institution{Macquarie University}\city{Sydney}\country{Australia}}
\email{xu.zhang12@hdr.mq.edu.au}
\author{Zhihong Cui}
\affiliation{\department{Department of Informatics}\institution{University of Oslo}\city{Oslo}\country{Norway}}
\email{zhihongc@ifi.uio.no}
\author{Yiqi Wang}
\affiliation{\institution{University of Southern Queensland}\city{Toowoomba}\country{Australia}}
\email{Yiqi.Wang@unisq.edu.au}
\author{Yi Zhang}
\affiliation{\institution{University of Technology Sydney}\city{Sydney}\country{Australia}}
\email{Yi.Zhang@uts.edu.au}
\renewcommand{\shortauthors}{Wu et al.}

\begin{abstract}
Tabular foundation models (TFMs) are being applied to Web data streams such as network intrusion detection, where
they adapt by updating the labelled rows of their context without retraining. When a past concept returns after
leaving the context, a TFM faces a \emph{recall-or-relearn decision} between restoring stored rows of the concept and
relearning it from new labels. We propose \method{} (\textbf{M}ixture of \textbf{I}n-\textbf{C}ontext
\textbf{E}xperts), a training-free learner that cuts the stream into segments at detected changes. It merges a closed
segment into a stored context when the TFM's estimated additional error on the segment is within a tolerance, under
a capacity limit on stored contexts. A fast level weights the stored contexts and
first-in-first-out (FIFO) windows by the log-loss of their predictions on the last labelled batch. We show that, with
pure stored contexts and a margin condition, this fast level gains over relearning, in expected accuracy, the
normalised positive area between the stored context's accuracy and the TFM's learning curve, up to a first-batch
term. A slow level mixes the fast level's prediction with that of the FIFO window of the last labelled rows and
bounds the discounted half-Brier regret of \method{} against this window at every time. Learning curves measured
beforehand predict the ordering of the observed recall gains. \method{} outperforms reset, which empties the context at detected
drifts, on grid streams with random boundary phases. Merging keeps fewer experts than pools that store every segment,
and with fixed segmentation a time-matched snapshot pool matches \method{} on the shifted grid. On real streams,
\method{} has the highest mean accuracy of all policies we compare.
\end{abstract}


\maketitle
\parfillskip=0pt plus 0.5\columnwidth\relax
\raggedbottom
\setlength{\abovedisplayskip}{4pt plus 1pt}\setlength{\belowdisplayskip}{4pt plus 1pt}%
\setlength{\abovedisplayshortskip}{2pt plus 1pt}\setlength{\belowdisplayshortskip}{4pt plus 1pt}%
\setlist{topsep=2pt,itemsep=1pt}

\section{Introduction}
\label{sec:intro}

\begin{figure}[t]
  \centering
  \includegraphics[width=\columnwidth]{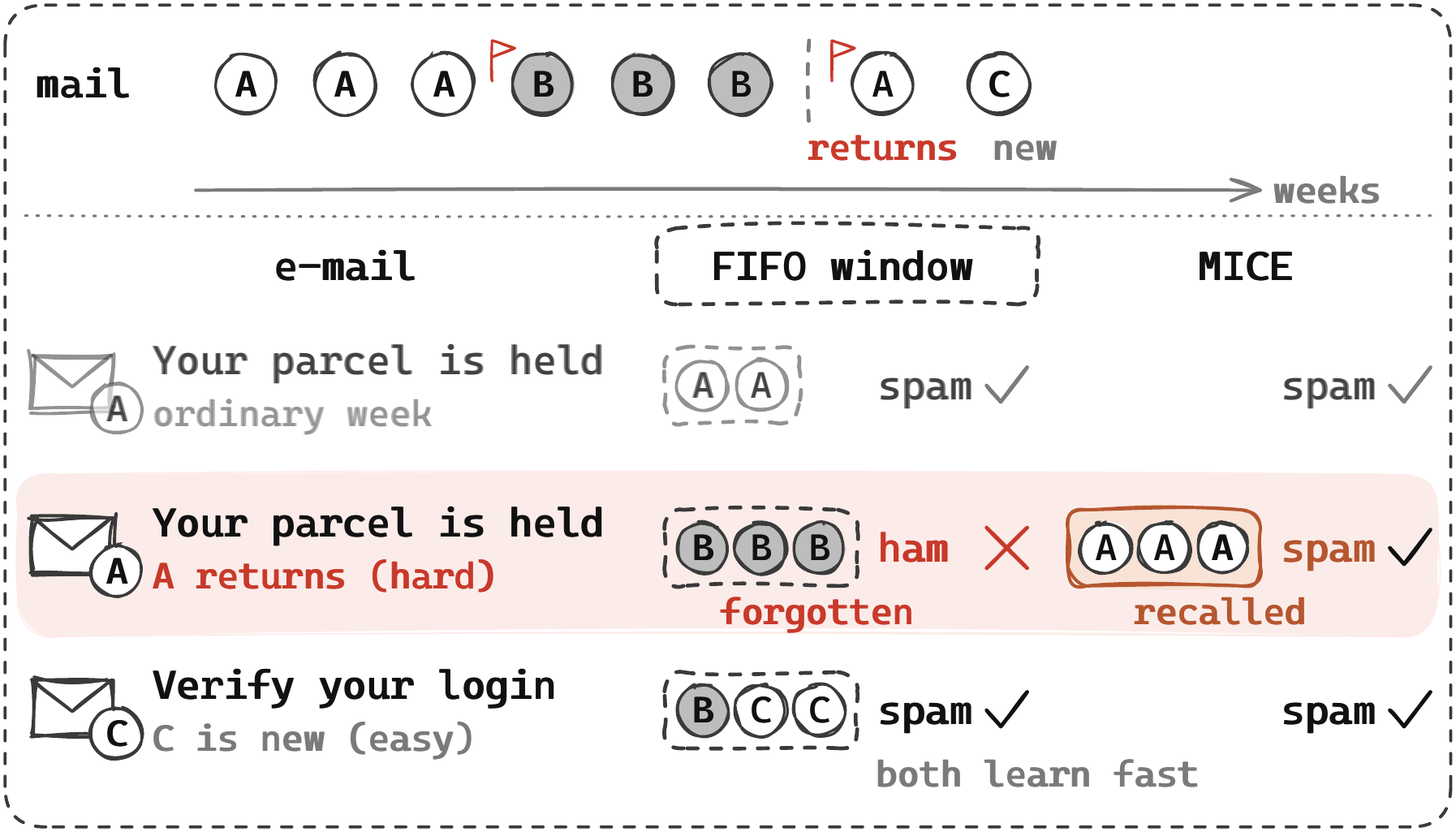}
  \caption{A frozen TFM filters spam from recent labelled mail, its context. Red flags mark detected changes. When
  campaign A returns, the FIFO window holds only B and predicts ham. \method{} recalls A's stored context and
  classifies A's mail as spam. Here, recall avoids relearning the hard campaign A. Both methods
  learn the new, easy campaign C quickly. Schematic.}
  \Description{The schematic mail timeline shows campaign A returning after B. A new campaign C follows the returning
  A. Red flags mark detected changes. Below the timeline, three rows compare the FIFO window with MICE. In an ordinary
  week both classify A's mail as spam. When A returns, the FIFO window holds only B's mail and predicts ham. MICE
  recalls a stored context of A and predicts spam. Both methods learn the new campaign C fast and predict spam.}
  \label{fig:problem}
\end{figure}

Web classification tasks such as fake news detection \citep{zhang2025fakenews}, time-series anomaly detection
\citep{wu2026wavae} and network intrusion detection receive their data as streams. Tabular foundation models (TFMs) have been applied to intrusion detection
\citep{garcia2025cybersecurity} and are being extended to data streams \citep{lourenco2026streams}. On a stream, a
TFM such as TabPFN \citep{hollmann2025tabpfn} adapts by updating the labelled rows of its context, without
retraining. A model may need new labels to rebuild a useful context when a returning pattern's rows have been
discarded. Recurring concepts therefore pose a \textbf{recall-or-relearn decision}: whether to restore stored rows or
rebuild the context from arriving labels. \Cref{fig:problem} illustrates the decision with a returning spam
campaign. Each batch is predicted before its labels arrive one batch later, and neither concept identities nor
change points are observed.

The first-in-first-out (FIFO) window follows recent data but discards the rows of past concepts, the contextual
knowledge loss of \cref{fig:intro}a. Keeping old rows alone does not repair this loss, because a context that
straddles an unknown boundary mixes conflicting rules and causes the conceptual interference of \cref{fig:intro}b. A
pure stored context can recover the predictions of a returning concept at once, and its benefit depends on how
quickly the TFM could learn the concept again. In our controlled tests, some concepts need 50 rows and others more
than 1000. For the former, one new batch can make the stored context redundant. For the latter, recall can avoid a
long recovery, as the learning curves in \cref{fig:intro}c illustrate. This difference makes the right answer to the
recall-or-relearn decision differ from concept to concept.

\begin{figure*}[t]
  \centering
  \includegraphics[width=\textwidth]{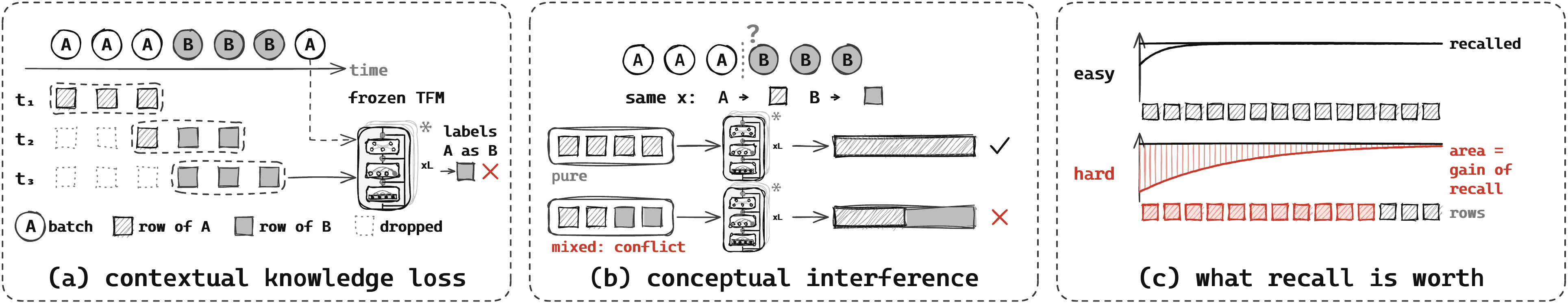}
  \caption{Three aspects of the recall-or-relearn decision. A and B are two concepts of one stream. (a) When A
  returns, the FIFO window has discarded A's rows and the TFM predicts A's batch with B's rule. (b) A context that
  straddles an unknown boundary mixes conflicting rules. (c) What recall is worth: the TFM relearns the easy concept
  from a few rows but the hard one only from many, and the red rows mark the relearning that recall avoids.
  Schematic. \Cref{sec:diag} measures each aspect. Under the purity and margin conditions of \cref{prop:recall}, the
  fast level's expected accuracy gain over relearning equals the normalised positive area in (c), up to a
  first-batch term.}
  \Description{Panel (a) shows a FIFO window discarding the rows of a concept that later returns. Panel (b)
  contrasts the predictions from a pure and a mixed context. Panel (c) compares the learning curves of an easy and a
  hard concept with stored-context accuracy, with the recall area shaded for the hard concept.}
  \label{fig:intro}
\end{figure*}

Existing stream policies either keep rows without grouping them by concept or reuse trained models. FIFO windows already make TFMs competitive with
adaptive tree ensembles \citep{lourenco2026streams}, and a prior trained on drifting causal models improves temporal
extrapolation \citep{helli2024drift}. Reset, which empties the context when the drift detection
method (DDM) fires \citep{gama2004ddm}, discards outdated rows but rebuilds the context even when a concept
has appeared before. The large tabular model (LTM) baseline of \citet{lourenco2026streams} pairs a short-term FIFO
buffer with a class-balanced long-term memory, without grouping rows by concept. Recurring-concept learners instead reuse trained
models \citep{goncalves2013rcd,gama2014recurrent,elwell2011nse,kolter2007dwm}. For a frozen TFM, a stored context can
serve as an expert directly. The challenge is to construct useful contexts from an unsegmented stream and decide
when their predictions deserve more weight than those of recent data.

Consider a fast level that weights stored contexts and recent windows by their loss on the last labelled batch.
Under pure-context and margin conditions, its expected accuracy gain over a learner that resets at concept boundaries
equals the area between stored-context accuracy and the relearning curve, up to a first-batch term. The area is large
when relearning is slow and vanishes when one batch already reaches stored-context accuracy. Learning curves measured
beforehand therefore predict which concepts offer the most benefit from recall. The online method needs no learning
curve at run time, because it selects contexts from their observed losses.

\method{} (\textbf{M}ixture of \textbf{I}n-\textbf{C}ontext \textbf{E}xperts) maintains stored contexts alongside
FIFO windows and combines their predictions using exponential weights \citep{littlestone1994weighted}.
\emph{Memory construction} closes segments at detected changes and merges them when their recall regret, the extra
0-1 error that the TFM makes on a segment when conditioned on a stored context instead of on the segment itself, is
small. This preserves reusable concepts while limiting the pool. \emph{Fast selection} weights the experts by the
log-loss of their predictions on the last labelled batch. When one expert leads every labelled batch of a
single-concept block by a log-loss margin, the regret of the fast level over the block is at most one batch of
clipped loss plus a term exponentially small in the margin. \emph{Slow aggregation} mixes the fast level's prediction
with the prediction of FIFO, weighted by their discounted half-Brier losses. It keeps discounted half-Brier regret
against FIFO within $2\ln2$ at every time, for any sequence of labels and expert predictions, which provides a loss
guarantee when following the pool does not yield a clear advantage.

We summarise our contributions as follows.
\begin{itemize}[leftmargin=*,topsep=2pt,itemsep=1pt]
  \item We formulate the recall-or-relearn decision for TFMs on streams with recurring concepts and measure its three
  aspects with a context stress test.
  \item We prove that, under pure-context and margin conditions, the expected accuracy gain of the fast level over
  relearning is a normalised learning-curve area up to a first-batch term, which lets learning curves measured before
  deployment predict where recall pays.
  \item We propose \method{}, a training-free mixture of in-context experts whose fast level selects stored contexts
  from recent losses and whose slow level keeps discounted half-Brier regret against FIFO at most $2\ln2$ at every
  time.
  \item We show that, with detected segmentation, \method{} is above reset on 70 of 72 grid streams with random
  boundary phases. On 29 real streams, where concepts are unobserved and experts often have similar losses, it has the
  highest mean accuracy of all policies and is above FIFO on 24.
\end{itemize}

\section{Related Work}
\label{sec:related}

Recurring-concept methods reuse trained models, whose knowledge of past concepts lives in weights that cost training to obtain. Drift detection and adaptation are surveyed by \citet{gama2014survey} and
\citet{lu2019learning}, and \citet{cabrera2025evolving} treat drift, forgetting and adaptation together.  \citet{goncalves2013rcd} select a stored model by statistical
tests. \citet{gama2014recurrent} learn a meta-model of each model's domain. Weighted ensembles reweight their members
as the stream changes \citep{elwell2011nse,kolter2007dwm}, and density-harmonized gradient descent learns imbalanced
streams online from actively queried labels \citep{huang2026density}. Dynamic TMoE adds and prunes trained experts at detected
regime shifts \citep{zhu2026dynamictmoe}, and drift-aware collaborative assistance combines trained experts across
heterogeneous streams \citep{yu2026drift}.  Online learning with experts tracks the best expert
through tracking and adaptive regret \citep{hazan2009efficient}, whatever the experts are, and \method{} uses this
view for contexts.

Prior-data fitted networks (PFNs) approximate posterior prediction from a labelled context
\citep{muller2022pfn,nagler2023statistical}, and TabPFN \citep{hollmann2025tabpfn}, TabICL \citep{qu2025tabicl} and
TabICLv2 \citep{qu2026tabiclv2} are TFMs of this kind with fixed weights. On streams, Drift-Resilient TabPFN
\citep{helli2024drift} retrains the prior to extrapolate temporal trends. The LTM baseline of
\citet{lourenco2026streams}, a TFM with a short-term FIFO buffer and a class-balanced long-term memory, outperforms
incremental ensembles. LTM does not recall a returning concept as a whole, because its two buffers do not group rows
by concept.

Recalling a stored context for a new interval needs a criterion for whether one context can serve another concept.
The $\mathcal H$-divergence \citep{bendavid2010theory} and the discrepancy distance \citep{mansour2009domain} measure
how well a source serves a target through the hypotheses it induces. Decision-based divergences
\citep{zhao2022comparing} define a distance between distributions by the decisions that their difference affects.
These divergences do not say when recalling a stored context beats relearning from the rows of the current concept,
because they compare distributions for a fixed hypothesis class. Our analysis uses learning curves to predict when
recall pays, and \method{} estimates the cost of reuse from the TFM's prediction errors.

\section{Setting and the Value of Context Memory}
\label{sec:def}

Our setting keeps the TFM frozen and lets a \emph{context policy} choose its rows, on a stream whose batches of $B$ rows are labelled one batch late. The TFM $f_\Theta(\cdot\mid\vx;\cC)$ predicts every batch from a context $\cC$ of at most $M$ labelled rows. The stream is a sequence of intervals $I_1,I_2,\dots$ with
distributions $P_1,P_2,\dots$ over $(\vx,y)$. Concept drift occurs between intervals in the sense of
\citet{lu2019learning}.

\begin{definition}[Recurring in-context concept drift]\label{def:recur}
Let $\cC_k$ be a context of rows of interval $I_k$ and $f_\Theta(\cdot\,;\cC_k)$ the $k$-th \emph{in-context expert}.
The drift into interval $I_{K+1}$ is \emph{recurring} if $\min_{k\le K}\Div_T(P_{K+1}\,|\,P_k)\le\delta$, where
$\Div_T(Q\,|\,P)$ is the population risk on $Q$ of the TFM conditioned on a context of $P$ minus that conditioned
on a context of $Q$, the recall regret of \cref{def:excess}, and
\emph{new} otherwise.
\end{definition}

The definition is model-oriented: a drift is recurring when some stored context serves the new interval with
additional risk at most $\delta$ relative to a context from that interval. Three policies are the reference points of
this paper. \emph{FIFO} conditions the TFM on the first-in-first-out window of the last $M$ labelled rows.
\emph{Reset} empties this window when DDM \citep{gama2004ddm} fires. The \emph{window ensemble} mixes FIFO windows of
several lengths. None of them keeps a library of past concepts.

\begin{figure*}[t]
  \centering
  \includegraphics[width=\textwidth]{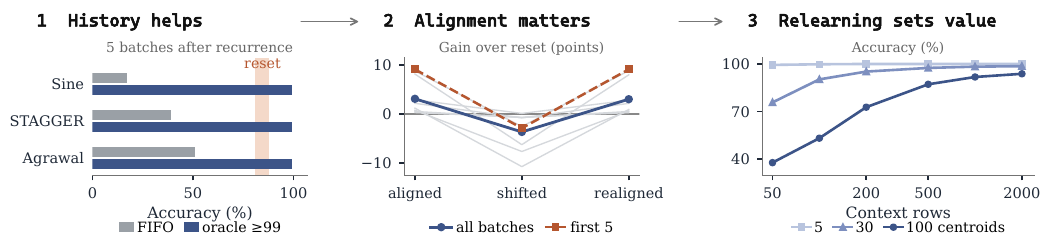}
  \caption{Context stress test. (1) Accuracy in the first five batches after a recurrence: FIFO (gray bars) and the
  oracle-context lower bound of 99\% (blue bars), with the reset range of 81--88\% shaded. (2) Gain of stored contexts
  over reset with aligned, shifted and realigned segments, over all batches (blue) and the first five batches after a
  recurrence (dashed orange), with single streams as pale traces. (3) Learning curves of concepts with 5, 30 and 100
  centroids.}
  \Description{Three connected diagnostic plots: paired horizontal bars compare FIFO and oracle contexts with a reset reference band, phase curves show loss and recovery of gains as segment alignment changes, and learning curves show accuracy versus context size for three concept complexities.}
  \label{fig:diag}
\end{figure*}

\label{sec:diag}The context stress test runs three tests with frozen TabPFN v2, which \cref{fig:diag} shows and
\cref{app:diag} specifies. The \emph{revisit test} shows that FIFO forgets a returning concept. In the five batches
after the return, FIFO reaches 0.17 to 0.51 accuracy on the Sine \citep{gama2004ddm}, STAGGER
\citep{schlimmer1986stagger} and Agrawal \citep{agrawal1993database} streams, whereas an oracle context of the
concept's past rows reaches at least 0.99. The \emph{splice test} shows that mixing concepts in a stored context can remove its gain. Moving the boundaries of eight grid streams 200 rows into fixed 500-row segments lowers the gain of stored
contexts over reset from $+3.09$ to $-3.67$ points, and realigning the segments on the same rows returns it to
$+3.01$. The \emph{curve test} shows that relearning speed differs across concepts. Concepts with 5 centroids are
learned from 50 rows, concepts with 100 centroids need more than 1000, and concepts of the standard generators about
200. Because 200 rows are two batches, resetting the context at a detected drift already recovers in the revisit
test. FIFO's failure on recurrence therefore establishes forgetting, whereas the learning curve determines whether
stored memory adds value beyond reset.

\section{Method: A Mixture of In-Context Experts}
\label{sec:method}

\method{} uses contexts as experts of one frozen TFM. A fast level weights stored contexts and FIFO windows by
the log-loss of their predictions on the last labelled batch. A slow level mixes the fast level's prediction with
that of FIFO and bounds the cost of following the fast level. \Cref{sec:m-recall} quantifies the gain of a stored
context. Segmentation and merging construct reusable contexts that hold mostly one concept.
\Cref{fig:method} gives the overview, \cref{alg:mice} in \cref{app:impl} the per-batch procedure, and
\cref{tab:notation} the notation. Formal statements are in \cref{app:theory} and proofs in \cref{app:proofs}.

\begin{figure*}[t]
  \centering
  \includegraphics[width=\textwidth]{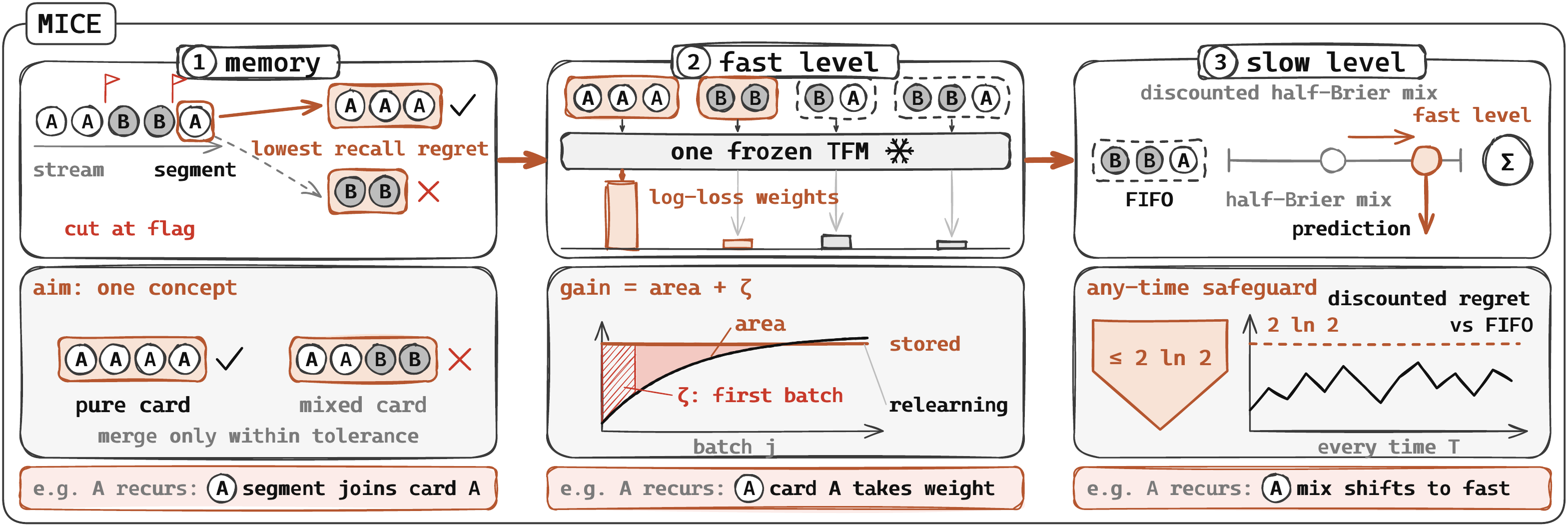}
  \caption{Overview of \method{} (schematic, bottom row illustrative). (1) Segments cut at detected changes join the stored context of lowest estimated recall regret within a tolerance. (2) The fast level weights stored-context and FIFO predictions from one frozen TFM by their log-loss on the last labelled batch. Under pure contexts and the margins of \cref{prop:recall}, its expected accuracy gain over relearning equals the area term plus the first-batch term $\zeta$. (3) The slow level mixes the fast-level and FIFO predictions. Its discounted half-Brier regret against FIFO is at most $2\ln2$ at every time.}
  \Description{Three columns, one per stage, are linked by arrows from memory to fast level to slow level. Column 1
  (memory): a stream A, A, B, B, A with red flags at detected changes is cut into segments. The last segment joins the
  stored context AAA, which has the lowest recall regret, and not BB. Its middle panel states the aim of one concept
  per context: a pure context AAAA is accepted and a mixed context AABB is rejected, and merges happen only within a
  tolerance. Column 2 (fast level): stored contexts and FIFO windows condition one frozen TFM and receive log-loss
  weights. Its middle panel plots accuracy against batch for a stored context and for relearning. The gain is the
  shaded area between the two curves plus a hatched first-batch term zeta. Column 3 (slow level): a discounted
  half-Brier mix between FIFO and the fast level gives the prediction. Its middle panel shows discounted regret
  against FIFO below 2 ln 2 at every time. An illustrative bottom row reads, for each stage when A recurs: the segment
  joins context A, context A takes the weight, and the mix shifts to the fast level.}
  \label{fig:method}
\end{figure*}

\subsection{Learning Curves Predict When Recall Pays}
\label{sec:m-recall}

Let concepts recur in blocks of $L$ rows, $J=L/B$ batches, and let $a_k(n)$ be the accuracy learning curve of concept
$k$, non-decreasing in $n$. On the $c$-th visit of concept $k$ its stored expert holds $n_c=\min(M,(c-1)L)$ rows.
The \emph{relearning} learner resets its context at the block boundary. The fast level weights the stored experts and
FIFO windows by the log-loss of their predictions on the last labelled batch, as specified in \cref{sec:m-fast}.
Here $K$ is the number of candidate experts, $\eta$ is the exponential-weight parameter, and $g$ is the
log-loss margin of the leading expert over every other expert, as defined in \cref{ass:margin}.

\begin{theorem}[Recall versus relearn]\label{prop:recall}
Let the stored context of concept $k$ hold rows of $k$ only, and on every labelled batch of the block let the expert
with the larger learning-curve value, the stored expert if $a_k(n_c)\ge a_k(\min(jB,M))$ and the window otherwise,
lead by a margin $g$ with $(K-1)e^{-\eta g}$ below the gap between its two largest predicted class probabilities
on every evaluated row. On the $c$-th visit of concept $k$, the expected
accuracy of the fast level minus that of relearning, averaged over the block, is
\begin{equation}\label{eq:area}
  \Delta_{k,c}(L)=\frac1J\sum_{j=1}^{J-1}\Bigl[a_k(n_c)-a_k\bigl(\min(jB,M)\bigr)\Bigr]_+ \;+\;\zeta_{k,c},
\end{equation}
where $\zeta_{k,c}$, the difference on the first batch of the block divided by $J$, satisfies $|\zeta_{k,c}|\le1/J$. We call the sum
the \emph{area term}.
\end{theorem}

The area term is the area by which the learning curve lies below the accuracy of the stored expert. It is zero on the
first visit and whenever the curve is flat, and at most $\tfrac{M}{L}\bigl(a_k(M)-a_k(B)\bigr)$ in general, as \cref{cor:flat} states. Thus slow relearning creates more room for recall, while long blocks dilute its average gain.
The identity covers the fast level with pure contexts under the stated margins. For the full method, whose
detected segments can contain mixed concepts and whose predictions include the slow level, the area term supplies
an empirical prediction tested in \cref{sec:rq1}.

\subsection{Constructing Reusable Contexts}
\label{sec:m-lib}

To recover the gain in \cref{eq:area}, a stored expert must retain rows that predict the returning concept well.

\method{} segments the stream at detected changes to keep stored contexts pure, because purity governs recall under the assumptions of \cref{prop:purity}. Under the posterior-predictive limit of \cref{ass:ppd}, \cref{prop:purity} shows that a stored context of two deterministic concepts with a shared input distribution recalls the concept that holds the majority of its rows. \method{} locates candidate changes using a two-sample test on the 300-row window's error rate.
Let $z_t$ be the number of errors that the window of the last 300 rows
makes on batch $t$, scored when its labels arrive. For every candidate change $b\in\{t-H_{\mathrm{cp}},\dots,t-1\}$,
the test compares the error rate on the batches $b,\dots,t-1$ ($n_2$ rows) with that on up to $R_{\mathrm{cp}}$
batches before $b$ ($n_1$ rows, at least two batches, none before the previous declared change), and declares a change
at the candidate with the largest excess
\begin{equation}\label{eq:cp}
  \frac{1}{n_2}\sum_{s=b}^{t-1}z_s-\frac{1}{n_1}\sum_{s=b-R_{\mathrm{cp}}}^{b-1}z_s-\tau(n_1,n_2),\quad
  \tau=\sqrt{\tfrac12\bigl(\tfrac1{n_1}+\tfrac1{n_2}\bigr)\ln\tfrac1{\delta_{\mathrm{cp}}}},
\end{equation}
if that excess is positive, with $H_{\mathrm{cp}}=3$, $R_{\mathrm{cp}}=5$ and $\delta_{\mathrm{cp}}=10^{-3}$.  When the
test declares a change at $b$, the open segment ends before batch $b$ and closes if it holds at least 300 rows.
A shorter segment stays in the windows only. The batches from $b$ on start the next segment. Without a declared
change, a segment closes after 500 rows once three further batches are labelled, and these three batches start the
next segment. A change can be declared before this closure, which gives segments of 300 to 700 rows. \emph{Fixed segmentation}, a variant used in several experiments,
closes a segment every 500 rows instead.

\method{} merges a segment into a stored context when the TFM's estimated additional error on the segment is within a tolerance, which compares contexts through the TFM's own predictions.
A closed segment $S$ joins the stored expert $\cC_k$ that minimises the \emph{recall regret}
\begin{equation}\label{eq:dt-hat}
  \widehat{\Div}_T(S\,|\,\cC_k)=\hat\epsilon_S(\cC_k)-\hat\epsilon_S(S),
\end{equation}
the 0-1 risk on $S$ of the TFM conditioned on $\cC_k$ minus that conditioned on $S$ itself, if this is at most
$\delta=0.02$. Otherwise $S$ founds a new stored expert. $\hat\epsilon_S(S)$ is cross-fitted in two folds because an
in-context learner memorises its context. \Cref{eq:dt-hat} is the finite-context estimate of the recall regret of
\cref{def:excess}. In the population limit of \cref{ass:ppd}, for two deterministic concepts with a shared input
distribution, this cost is their disagreement mass under the 0-1 loss, which is \cref{prop:purity} with $\pi=0$.
Concepts that conflict on more than a share $\delta$ of the inputs are therefore not merged. A stored expert keeps
its last $M$ rows, and the pool holds at most 12 (least recently used out).

\subsection{Selecting Contexts from Recent Feedback}
\label{sec:m-fast}

Because a cumulative loss still carries the previous concept after a recurrence, the fast level weights the experts by
their log-loss on the last labelled batch only. The experts are the FIFO windows of the last 100, 300 and $M$
labelled rows and the stored experts of the pool, and the fast level weights the $K$ experts by
\begin{equation}\label{eq:fast}
  w_k\propto\exp(-\eta\,\ell_k),\qquad \eta=10,
\end{equation}
where $\ell_k$ is the mean log-loss of the predictions of expert $k$ on the last labelled batch. In other words, a stored expert that leads the last labelled batch by a clear log-loss margin takes most of the
weight one batch after its concept returns. If one
expert leads the last labelled batch by a margin $g$ in log-loss as in \cref{ass:margin}, it holds weight at least
$1/(1+(K-1)e^{-\eta g})$, and the mixture predicts its class on every row where its top two probabilities differ by
more than $(K-1)e^{-\eta g}$, as \cref{lem:margin} shows. Over a block of one concept on which the margin holds at every batch,
with log-losses clipped at $c$, the regret of the fast level is at most $c$ for the first batch, because labels arrive
one batch late, plus a term exponentially small in $g$, as \cref{prop:block} states.

\subsection{Bounding the Loss of Following Memory}
\label{sec:m-slow}

When the experts have similar losses, the fast level can change its weights in response to small batch-to-batch
fluctuations. The slow level uses a discounted loss history to control the cost of following the fast level.
\method{} mixes two experts, FIFO ($k=0$) and the fast level \eqref{eq:fast} ($k=1$), with predictions $p_{t,0}$ and
$p_{t,1}$ on batch $t$. With $\ell_{t,k}\in[0,1]$ the mean half Brier score $\tfrac12\lVert p-e_y\rVert_2^2$ of the
predictions of expert $k$ on batch $t$,
\begin{align}
  L_{t,k}&=\gamma L_{t-1,k}+\ell_{t,k},\qquad L_{0,k}=0, \label{eq:outer-loss}\\
  v_{t,k}&\propto\exp\bigl(-\eta_2\,\gamma\,L_{t-1,k}\bigr), \label{eq:outer-weight}\\
  p_t&=v_{t,0}\,p_{t,0}+v_{t,1}\,p_{t,1}, \label{eq:outer}
\end{align}
with $\gamma=0.9$ and $\eta_2=\tfrac12$. Because the discount spans about $1/(1-\gamma)$ batches, the slow level moderates the influence of the fast
predictions when batch-to-batch losses fluctuate.

\begin{theorem}[Any-time safeguard]\label{prop:outer}
Let $K$ experts be combined by \eqref{eq:outer-loss}--\eqref{eq:outer}, with $\ell$ the half Brier score and a step
$\eta_2\le\tfrac12$. For every
sequence of labels and expert predictions, every time $T$ and every expert $k$,
\begin{equation}\label{eq:outer-bound}
  \sum_{t=1}^{T}\gamma^{\,T-t}\bigl(\ell_t(p_t)-\ell_{t,k}\bigr)\;\le\;\frac{\ln K}{\eta_2} .
\end{equation}
\end{theorem}

With $K=2$, the discounted regret of \method{} against FIFO, and against the fast level, is at most
$2\ln2\approx1.39$ at every time, over a discounted horizon of $1/(1-\gamma)=10$ batches, with no assumption on the
stream. The bound covers the method as run, not the variant with the log-loss or the variant with $\eta_2=10$ in
\cref{tab:main}.

\subsection{Timing and Cost}
\label{sec:m-cost}

When the labels of the previous batch arrive, every loss scores a prediction issued before they arrived, and the
contexts are updated only afterwards. No expert is therefore scored on rows it already holds, and the slow level is
scored on the predictions actually issued. At most one segment closes
per batch, and \cref{app:impl} gives the initialisation.  The cost is one TFM call per expert and batch, at most 15. Closing a segment adds two cross-fitted calls and one call per stored expert, at most 14 calls, which brings a closing batch to at most 29 calls. The change test reuses the predictions of the 300-row window and makes no call.

In summary, stored contexts provide the gain of recall, and the slow level bounds what following them can cost.

\section{Experiments}
\label{sec:exp}

The evaluation follows the gain predicted by \cref{eq:area} from ideal contexts to online use.
RQ1: does the area identity describe recall gains under its assumptions, and predict their ordering for the
full method? RQ2: can detected segmentation recover gains when boundaries are unknown?
RQ3: how does slow aggregation affect prediction loss and accuracy on real streams?
RQ4: which gains come from the pool and its organisation, and what resources do they require?

\subsection{Experimental Setup}
\label{sec:setup}

Every policy is evaluated prequentially, predicting each batch of $B=100$ rows before its labels arrive one batch later. Each TFM expert conditions on at most $M=1000$ context rows and the pool holds at most 12 stored contexts. The backbone is TabPFN v2 with four ensemble members \citep{hollmann2025tabpfn}, and FIFO windows exceed 1000 rows only in the larger-window comparison of \cref{sec:ablation}. GPU experiments run on eight NVIDIA A800 GPUs, while the trained stream learners run on CPUs. 

The streams range from a controlled grid with known concepts and measured learning curves to real streams, where both are unobserved. The grid mixes Gaussian centroids in ten dimensions with four classes, using centroid counts from $\{5,30,100\}$. Each of the three concepts assigns its own class to every centroid, keeping $P_X$ fixed while creating conflicts wherever classes differ. Grid streams have blocks of 500 or 2000 rows and three cycles, as \cref{app:grid} details.
Thirty streams come from six standard generators with recurring concepts in River
\citep{montiel2021river}\footnote{\url{https://riverml.xyz}}: SEA \citep{street2001sea}, STAGGER
\citep{schlimmer1986stagger}, Agrawal \citep{agrawal1993database}, Sine \citep{gama2004ddm}, RBF \citep{bifet2009new}
and Hyperplane \citep{hulten2001mining}. Generator streams have blocks of 2000 rows and three cycles over three or four concepts. Twenty-nine real streams come from 13 sources, the long ones in segments of
100{,}000 rows. They are Electricity from River, seven INSECTS streams
\citep{souza2020insects}\footnote{\url{https://sites.google.com/view/uspdsrepository}}, Forest
CoverType\footnote{\url{https://archive.ics.uci.edu/dataset/31/covertype}}, and PokerHand, Airlines, Rialto
\citep{losing2016knn}, NOAA Weather \citep{elwell2011nse}, Spam \citep{katakis2010tracking} and Phishing from a public
collection of drift data sets\footnote{\url{https://github.com/ogozuacik/concept-drift-datasets-scikit-multiflow}},
which contributes its remaining streams with at least 5000 rows and at most ten classes. Six later segments of Airlines, CoverType and PokerHand serve as confirmation
segments, and the earlier real-stream sets served as development data, as \cref{app:prereg} details. Nine
independent confirmation segments come from two further sources, Bitcoin Heist ransomware transactions
\citep{akcora2020bitcoinheist}\footnote{\url{https://archive.ics.uci.edu/dataset/526}} and CIC-IDS2017 network
intrusions \citep{sharafaldin2018cicids}\footnote{\url{https://www.unb.ca/cic/datasets/ids-2017.html}}.

The baselines cover the two ways of following drift, adapting the context of a TFM and training a model on the stream. The in-context baselines are FIFO, LTM \citep{lourenco2026streams}, reset \citep{gama2004ddm}
and the window ensemble of the three FIFO windows of \method{} without the pool. Trained stream learners are Adaptive
Random Forest (ARF) \citep{gomes2017arf}, Streaming Random Patches (SRP) \citep{gomes2019srp}, Hoeffding Adaptive Tree
(HAT) \citep{bifet2009hat}, Leveraging Bagging \citep{bifet2010levbag}, ADWIN Bagging \citep{bifet2009new}
and Learn++.NSE \citep{elwell2011nse}.  The incremental learners use the default settings of \texttt{river} \citep{montiel2021river}
with ten base models and the same batch protocol. Online forecasters FSNet \citep{pham2023fsnet}, OneNet \citep{zhang2023onenet}, ShifTS
\citep{zhao2026shifts} and ADAPT-Z \citep{huang2026adaptz} get a softmax head and are tuned on the five development streams.
Drift-Resilient TabPFN \citep{helli2024drift} runs without tuning.

Four metrics complement accuracy by weighting classes equally and by scoring the predicted probabilities. F1 is the macro-F1 over the classes that occur in a stream. AUC, the area under the receiver
operating characteristic curve, is computed one-vs-rest from each policy's predicted class probabilities and averaged
over these classes. 
The expected calibration error (ECE) bins the predictor's largest class probability into 15 equal bins. Brier is the
half Brier score $\tfrac12\lVert p-e_y\rVert_2^2$ averaged over rows. All four are reported in percent.

Unless a result names fixed segmentation, \method{} segments at detected changes, which improves on fixed
segmentation by $+0.04$ points on the aligned grid and $+0.01$ on the real streams. The generator streams, the
confirmation segments, the ablations and the timings use fixed segmentation, and each comparison below names its
variant. \Cref{app:prereg} records how the development and confirmation streams were separated and every registered
hypothesis with its outcome. \Cref{tab:main} reports both the slow step $\eta_2=\tfrac12$ covered by
\cref{prop:outer} and the step of 10.

\begin{table*}[t]
  \centering
  \caption{Main results (\%, mean over streams). Worst $\Delta$: lowest accuracy difference to FIFO over the real
  streams, in points. Indep.\ (9): the nine independent confirmation segments of Bitcoin Heist and CIC-IDS2017. The top three values per column are
  shaded.}
  \label{tab:main}
  \footnotesize
  \setlength{\tabcolsep}{2.5pt}
  \begin{tabular*}{\textwidth}{@{\extracolsep{\fill}}lcccccc|ccccc|c|c}
    \toprule[1pt]
     & \multicolumn{6}{c|}{\textbf{Real streams (29)}} & \multicolumn{5}{c|}{\textbf{Controlled grid (24)}} & \textbf{Gen.\ (30)} & \textbf{Indep.\ (9)} \\
    \cmidrule(lr){2-7}\cmidrule(lr){8-12}\cmidrule(lr){13-13}\cmidrule(lr){14-14}
    \textbf{Policy} & \textbf{Acc.}\ $\uparrow$ & \textbf{F1} $\uparrow$ & \textbf{AUC} $\uparrow$ & \textbf{ECE} $\downarrow$ & \textbf{Brier} $\downarrow$ & \textbf{Worst} $\Delta$ $\uparrow$
           & \textbf{Acc.}\ $\uparrow$ & \textbf{F1} $\uparrow$ & \textbf{AUC} $\uparrow$ & \textbf{ECE} $\downarrow$ & \textbf{Brier} $\downarrow$ & \textbf{Acc.}\ $\uparrow$ & \textbf{Acc.}\ $\uparrow$ \\
    \midrule
\multicolumn{14}{l}{\emph{Trained stream learners}} \\
ARF~\citep{gomes2017arf} & 73.4 & 58.4 & 85.4 & 6.0 & 18.7 & $-$37.5 & 67.7 & 65.6 & 83.8 & 11.2 & 22.9 & 88.0 & 98.39 \\
SRP~\citep{gomes2019srp} & 76.6 & 62.3 & 87.5 & 6.0 & 16.7 & $-$19.8 & 69.5 & 68.1 & 84.7 & 8.8 & 21.6 & 86.4 & 98.57 \\
HAT~\citep{bifet2009hat} & 64.3 & 51.1 & 79.1 & 16.5 & 27.4 & $-$49.6 & 51.4 & 49.4 & 74.5 & 13.3 & 32.1 & 79.2 & 97.93 \\
Leveraging Bagging~\citep{bifet2010levbag} & 72.4 & 57.3 & 85.9 & 5.1 & 19.1 & $-$34.6 & 62.0 & 60.6 & 82.4 & 7.2 & 24.9 & 85.9 & 98.39 \\
ADWIN Bagging~\citep{bifet2009new} & 65.1 & 51.5 & 81.4 & 7.8 & 24.0 & $-$50.1 & 51.2 & 49.2 & 76.1 & 9.4 & 30.8 & 77.0 & 97.91 \\
Learn++.NSE~\citep{elwell2011nse} & 67.7 & 51.2 & 82.8 & 8.5 & 22.3 & $-$41.0 & 61.6 & 60.0 & 81.2 & 30.0 & 33.6 & 84.1 & 94.86 \\
FSNet~\citep{pham2023fsnet} & 77.3 & 64.9 & 88.1 & 4.6 & 15.6 & $-$26.7 & 62.7 & 61.3 & 80.4 & 8.4 & 24.8 & 87.5 & 98.21 \\
OneNet~\citep{zhang2023onenet} & 77.8 & 65.1 & 88.1 & 4.2 & 15.2 & $-$22.1 & 62.4 & 60.9 & 80.1 & \second{7.0} & 24.6 & 87.6 & 97.19 \\
ShifTS~\citep{zhao2026shifts} & 73.4 & 59.9 & 84.6 & 3.5 & 17.7 & $-$30.4 & 51.1 & 47.0 & 71.7 & \first{6.0} & 29.3 & 76.4 & 97.72 \\
ADAPT-Z~\citep{huang2026adaptz} & 65.6 & 50.7 & 80.1 & 5.6 & 22.3 & $-$52.3 & 50.1 & 46.2 & 70.7 & \first{6.0} & 30.0 & 78.1 & 97.43 \\
\midrule
\multicolumn{14}{l}{\emph{In-context policies, frozen TabPFN}} \\
FIFO & \second{83.3} & 73.5 & 89.9 & 1.7 & \third{11.6} & -- & 61.9 & 60.9 & 83.7 & 16.7 & 25.1 & 87.6 & 98.57 \\
LTM~\citep{lourenco2026streams} & 82.6 & 72.1 & \first{91.0} & 1.5 & 11.9 & $-$2.7 & 62.2 & 61.4 & 84.8 & 16.0 & 24.5 & 87.6 & 98.21 \\
Reset~\citep{gama2004ddm} & 81.5 & 71.3 & 88.5 & 2.6 & 12.9 & $-$9.6 & 84.1 & 83.8 & 91.4 & 7.3 & 13.1 & 94.8 & 98.49 \\
Window ensemble & \third{83.0} & 73.0 & 89.5 & \second{1.2} & 11.7 & $-$1.4 & 81.6 & 81.0 & 90.8 & 8.7 & 14.9 & 94.6 & 98.44 \\
Drift-Resilient TabPFN~\citep{helli2024drift} & 81.7 & 68.7 & 88.4 & 3.0 & 12.9 & $-$4.2 & 76.4 & 75.6 & 87.4 & 9.8 & 17.9 & 93.7 & 98.37 \\
\midrule
\multicolumn{14}{l}{\emph{Ours}} \\
\textbf{\method{} (full, detected seg.)} & \first{83.8} & \first{74.4} & \second{90.7} & \third{1.3} & \first{11.2} & \second{$-$0.23} & \third{86.4} & \third{86.0} & \second{92.4} & 9.9 & \third{12.6} & \third{95.0} & \third{98.59} \\
\method{}, fixed segments & \first{83.8} & \second{74.3} & \third{90.5} & \third{1.3} & \first{11.2} & \third{$-$0.25} & \third{86.4} & \third{86.0} & \third{92.3} & 10.1 & 12.7 & \third{95.0} & \second{98.60} \\
\method{}, fixed segments, $\eta_2=10$ & \first{83.8} & \third{74.2} & \third{90.5} & \first{1.1} & \first{11.2} & \first{$-$0.08} & \second{86.6} & \second{86.3} & \first{93.6} & \second{7.0} & \second{11.5} & \second{95.2} & \first{98.61} \\
\method{} w/o slow level (fixed seg.) & \second{83.3} & 73.4 & 90.3 & 1.5 & \second{11.5} & $-$2.3 & \first{86.8} & \first{86.5} & \first{93.6} & \third{7.1} & \first{11.4} & \first{95.3} & 98.50 \\
    \bottomrule[1pt]
  \end{tabular*}
\end{table*}

\subsection{RQ1: Learning Curves Predict Recall Gains}
\label{sec:rq1}

The oracle-boundary test checks the identity of \cref{prop:recall} with known boundaries, the relearning learner and
the stored expert of the theorem, which holds the visit-specific $n_c$ rows. Its observed quantity is the measured area
term $J^{-1}\sum_{j=2}^{J}(\mathrm{acc}_{X,j}-\mathrm{acc}_{R,j})$, where $\mathrm{acc}_{X,j}$ and $\mathrm{acc}_{R,j}$
are the accuracies of the fast level and of relearning on batch $j$ of the block. The first batch enters only
$\zeta_{k,c}=(\mathrm{acc}_{X,1}-\mathrm{acc}_{R,1})/J$. Its predicted quantity is the area term computed from learning
curves with the same $n_c$. On seeds 40 and 41 the two lie within 0.7 points of each other in every cell, with a mean
absolute difference of 0.37 points and Spearman correlation 1.0, as \cref{fig:rq1}b shows. The mixture assigns more
than half of its weight to the more accurate expert on every batch where the two differ by more than 10 points.

Beyond the ideal setting, learning curves measured before the grid streams were run also rank the observed recall
gains, as \cref{app:grid} describes. These registered predictions insert the curves into the area term with the full stored
context of $M$ rows on every recurrence. Their observed counterpart is the gain over reset of a complete policy with
fixed segmentation. On seeds 10 and 11, whose concepts supplied the curves, \method{} w/o slow level exceeds reset on
all eight streams with 30 or 100 centroids, more for shorter blocks and harder concepts. Predicted and observed gains
have Spearman correlation 0.94 across the six cells (10.4 against 8.6 points in the hardest cell). For the full method
on fresh seeds 20 and 21 the correlation is 0.83, as \cref{fig:rq1}a shows. With the visit-specific $n_c$ in place of
$M$, the predictions change only for $L=500$ (100 centroids: 9.2 instead of 10.4), and the rank correlations stay
unchanged.

\begin{figure}[t]
  \centering
  \includegraphics[width=\columnwidth]{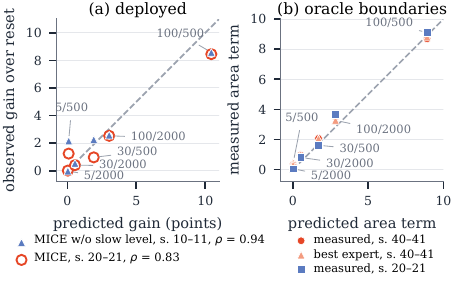}
  \caption{Learning curves predict the gain of recall, per grid cell (centroids/block length). Dashed lines mark
  equality. (a) With fixed segments, observed gain over reset of \method{} w/o slow level (seeds 10, 11) and of \method{} (seeds 20, 21) against
  the prediction from learning curves measured beforehand. (b) With oracle boundaries, the measured area term
  of \cref{prop:recall} (seeds 40, 41 and 20, 21) and the best expert per batch against the predicted area.}
  \Description{Scatter plots compare predicted recall gains with observed gains and compare the predicted area term with its oracle-boundary measurement.}
  \label{fig:rq1}
\end{figure}

Stored contexts add more accuracy on slowly learned concepts, because they replace a longer period of relearning. With fixed segments, \method{} exceeds reset and the window ensemble on all eight streams with 30 or 100 centroids. Averaged over the twelve streams of seeds 20 and 21 it reaches 86.5 against 84.3 for reset, 81.8 for the window ensemble and 62.4 for FIFO, up to
8.5 points over reset in the hardest cell. On seeds 10 and 11 the trained learners ARF and SRP reach
68.0 and 69.7, against 83.9 for reset and 86.2 for \method{} with fixed segments. On the six standard generators,
whose learning curves are flat after about 200 rows, the predicted gain is at most 0.4 points and the observed gain
over reset is 0.1 points with fixed segments (computed from unrounded means, 95.0 against 94.8 after rounding): where relearning is cheap, memory has little to add, as \cref{cor:flat} predicts.

\subsection{RQ2: Recovering Gains with Unknown Boundaries}
\label{sec:rq2}

Detected segmentation preserves recall gains without knowing the boundary phase. On 72 new
grid streams with random boundary offsets, \method{} is above reset on 70, by $+2.85$
points on average: $+5.17$ with 500-row blocks (95\% interval $[+4.59,+5.72]$), $+1.48$ with 2\,000-row blocks and
$+1.91$ with variable blocks, against $+0.66$, $+0.74$ and $+0.65$ for fixed segmentation, as \cref{fig:segments}b shows.
On 72 grid confirmation streams from seeds 50 to 53, \method{} is above reset on all 72, by $+2.69$ points. Detected segmentation recovers what
fixed segmentation lost, $+6.55$ points on the shifted grid and $+1.74$ on the misaligned schedules, and leaves the
aligned grid and the real streams unchanged, as \cref{fig:segments}a shows.

\begin{figure*}[t]
  \centering
  \includegraphics[width=\textwidth]{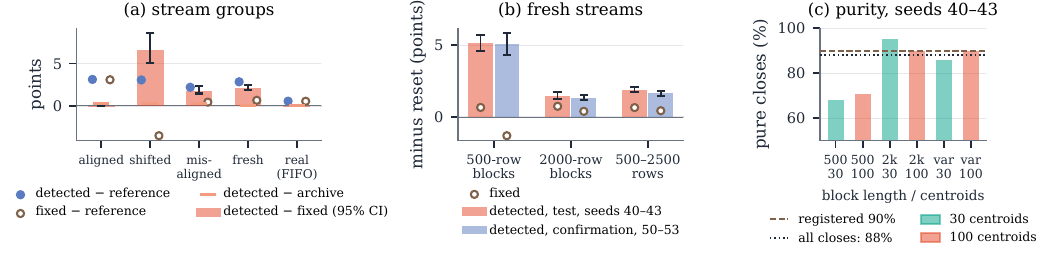}
  \caption{Detected against fixed segmentation. (a) Their difference per stream group. (b) Gain over reset on fresh
  streams with random boundary phases, test and confirmation seeds. (c) Share of segment closes that are at least 95\%
  pure. The dashed line marks the target share of 90\%.}
  \Description{Panels compare detected and fixed segmentation across stream groups, recall gains on test and confirmation streams, and segment purity.}
  \label{fig:segments}
\end{figure*}

\begin{figure*}[t]
  \centering
  \includegraphics[width=\textwidth]{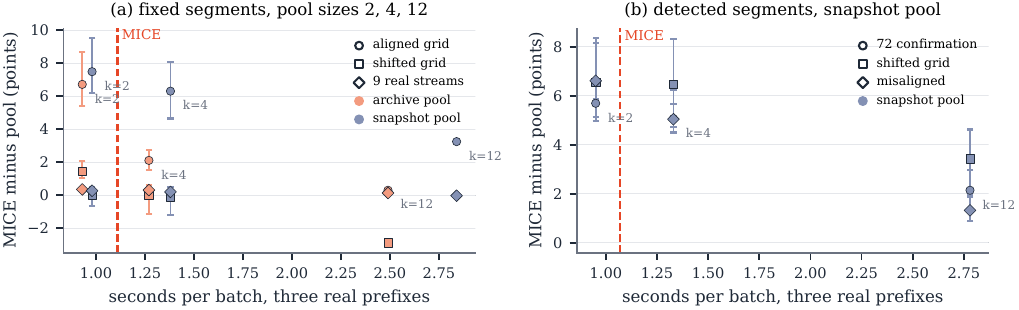}
  \caption{\method{} minus unmerged pools (archive, snapshot) of sizes 2, 4 and 12 against their time per batch. The
  dashed line is the time of \method{}. (a) Fixed segmentation. (b) Detected segmentation.}
  \Description{Accuracy differences between MICE and archive or snapshot pools are plotted against time per batch for fixed and detected segments.}
  \label{fig:pools}
\end{figure*}

The alignment diagnostic explains this recovery, because boundaries inside fixed segments mix concepts and reduce
recall gains. When the boundaries of the eight aligned grid streams move 200 rows into the segments, the gain of fixed
segmentation over reset falls by 6.8 points on average, and it is below reset on 7 of the 8 streams. On twelve
streams with misaligned block schedules it is above reset on 8, by $+0.47$ points, as test~2 of \cref{fig:diag} shows
and \cref{fig:phase} details per stream. Restarting the segments at the known phase of the boundaries restores the
gain ($+3.01$ against $+3.09$ points), which shows that the loss lies in the phase of the segments.

Segment purity matters because a late detection mixes the rows of two concepts in one segment. The detector declared 570 of the 576 changes of the test streams and 575 of the 576 of the confirmation
streams, each in the batch that holds the boundary or the next one. Every stored context was majority-pure, and 88\% of the segment closes held at
least 95\% rows of one concept (87\% on the confirmation streams). The impure closes concentrate on 500-row
blocks, as \cref{fig:segments}c shows.

\subsection{RQ3: The Cost of Following Memory on Real Streams}
\label{sec:rq3}

Real-stream concepts and learning curves are unobserved, and the experts often have similar prediction losses. On the
grid the leading expert holds weight 1.000 from the second batch of a recurring block. On the real streams the median
margin, the gap in mean log-loss between the leading and the next-best expert on a labelled batch, is about 0.01
nats, the regime the slow level is built for. With detected segmentation, \method{} is never more than 0.23 points
below FIFO across the 29 real streams and is above it on 24 ($+0.58$ on average). With fixed segmentation, the slow
level limits the largest shortfall of \method{} to 0.25 points, against 2.3 for \method{} w/o slow level. In
comparison, reset loses up to 9.6 points to FIFO and 1.8 on average, and the window ensemble up to 1.4, as
\cref{tab:main} reports and \cref{app:perstream} gives per stream. The largest gains occur on Rialto ($+4.7$, a daily cycle) and
the INSECTS streams ($+0.2$ to $+1.9$). In \cref{tab:main}, the trained baselines have lower mean accuracy on these
streams. Beyond accuracy, \method{} has the highest F1 and a lower Brier score than every baseline on the real
streams. LTM has the highest mean AUC, with its advantage concentrated on the PokerHand segments, where its
class-balanced buffer keeps rare classes.

The confirmation segments of Airlines, CoverType and PokerHand, run with fixed segmentation, reproduce the real-stream
result. \method{} is never more than 0.01 points below FIFO on any of the six segments and is not below reset. It is
more accurate than all four in-context baselines on five and has the highest mean (85.3 against 85.1 for FIFO). The
slow level turns the $-2.0$ and $-1.9$ points of \method{} w/o slow level on the two PokerHand segments into $+0.4$
and $+0.5$. The one-sided 95\% block-bootstrap bound of the difference to FIFO is above $-0.3$ points on 28 of the 29
real streams. On Spam, with 61 batches, non-inferiority is undecided. 

On the nine independent confirmation segments of Bitcoin Heist and CIC-IDS2017, detected segmentation keeps the
accuracy of fixed segmentation, from which the mean of \method{} differs by less than 0.01 points. On every segment
the one-sided 95\% bootstrap lower bound on its accuracy difference to FIFO is at least $-0.25$ points, and its
observed accuracy is never more than 0.18 points below the better of FIFO and reset. It is above its pool-free replay
on eight of nine segments. Mean accuracies over the nine segments differ by less than 0.4 points across the
in-context policies, at about 98.5\%, and \method{} is the most accurate of the in-context policies on four of the
nine segments.

Replacing the half Brier score with the log-loss can harm the slow level, because one confidently wrong batch after a
switch outweighs the next ten. With the log-loss in place of the half Brier score, the accuracy of \method{} on one
grid stream falls from 86.4 to 74.9. On the grid with fixed segments, the step $\eta_2=\tfrac12$ covered by
\cref{prop:outer} has an ECE 3.1 points above the step of 10. The largest discounted half-Brier regret of
\method{} against FIFO over 33 development streams is 0.15, within the bound of 1.39. 

\subsection{RQ4: Memory Organisation and Resource Tradeoffs}
\label{sec:ablation}

Removing stored contexts lowers accuracy on all 24 grid streams (by 5.6 points on average) and on 28 of the 29 real
streams (by 0.35) in this fixed-segmentation ablation. Without them, the two levels mix only the three FIFO windows.

 Merging keeps fewer experts than two pools that never merge and share the windows and two-level rule of \method{}
with a pool size of 12. An \emph{archive} keeps every closed segment as its own expert, and a \emph{snapshot} pool
stores the latest 1000 rows at every close. With fixed segments, \method{} is above the archive on the aligned grid by
$+0.30$ points and within 0.1 points of both pools on the real streams, with fewer experts (8.4 against 14.2 per
batch) and TFM calls (9.7 against 14.0).

Capped to the time per batch of \method{}, both
fixed-segmentation pools are less accurate than \method{}, by 6.7 and 7.5 points on the aligned grid and 0.35 and
0.27 on nine real streams. On the shifted grid, the time-matched snapshot pool is level with \method{}, and the uncapped archive is ahead of it, as \cref{fig:pools}a shows.

With detected segments on the fresh grid streams, \method{} keeps the pool at a third of the archive's experts and is
0.11 points more accurate. On the grid confirmation streams it exceeds a snapshot pool of equal capacity by 2.14
points and is level with it only on 2\,000-row blocks, as \cref{fig:pools}b shows. On 144 pairs of grid segments the recall
regret separates same from different concepts with AUROC 1.00.  

Runtime and peak GPU memory grow with the rows that a policy conditions on, and \method{} is compared with FIFO
windows of up to 13\,400 rows. \method{} with fixed segments takes 1.07\,s per batch on three real prefixes on one
GPU, against 0.22\,s for FIFO with 1000 rows, 0.42\,s with 3000 rows and 1.91\,s with 13\,400, as plotted in
\cref{app:perstream}. Peak GPU memory is 0.8 to 26.4\,GB for FIFO and 1.0\,GB for \method{}. A 3000-row window, about 2.5 times
cheaper, is more accurate than fixed-segmentation \method{} on 21 of the 29 real streams. \method{} is ahead by
$+0.31$ points on average because of the sharply drifting streams, four INSECTS streams by $+3.3$ to $+5.5$ and
Rialto by $+6.9$, as \cref{fig:fifo3000} details per stream. These comparisons cover complete policies at different context budgets, whereas \cref{prop:outer} compares
\method{} with FIFO at $M$ rows.

\section{Conclusion}
\label{sec:conclusion}

A stored context lets a frozen TFM reuse a past concept without retraining. Under pure-context and margin
conditions, its value is the relearning it saves, a normalised learning-curve area up to a first-batch term. \method{}
builds contexts at detected changes and weights them by their last-batch log-loss. Its slow level keeps discounted
half-Brier regret against FIFO at most $2\ln2$ at every time. Learning curves predict the ordering of recall gains on
the grid. Detected segmentation recovers these gains under unknown boundary phases. On 29 real streams, \method{} has
the highest mean accuracy of all policies and is above FIFO on 24. Learning curves measured before deployment can
therefore indicate whether stored contexts will pay.
\clearpage
\bibliographystyle{ACM-Reference-Format}
\bibliography{refs,refs_r4}

\appendix

\section{Theory}
\label{app:theory}

\begin{table}[t]
  \centering
  \caption{Notation.}
  \label{tab:notation}
  \footnotesize
  \begin{tabular}{@{}lp{0.72\columnwidth}@{}}
    \toprule[1pt]
    \textbf{Symbol} & \textbf{Meaning} \\
    \midrule
    $(\vx,y)$ & a row and its label \\
    $B$, $M$ & batch size (100 rows), context budget (1000 rows) \\
    $\cC$, $\cC_k$ & a context, the stored context $k$ \\
    $f_\Theta(\cdot\mid\vx;\cC)$ & the frozen TFM \\
    $P$, $Q$ & concepts \\
    $\epsilon_P(\cC)$ & risk on $P$ of the TFM given $\cC$ \\
    $\Div_T$ & recall regret \\
    $S$, $\delta$ & a closed segment, merge threshold (0.02) \\
    $H_{\mathrm{cp}}$, $R_{\mathrm{cp}}$, $\delta_{\mathrm{cp}}$, $\tau$ & constants and margin of the change test \\
    $\pi$, $\nu$ & majority share of a context, distance of a declared change to its boundary \\
    $K$ & number of experts \\
    $\ell_k$, $\eta$, $w_k$ & last-batch loss, step (10) and weight of the fast level \\
    $g$ & margin of the leading expert \\
    $L$, $J$ & block length in rows and in batches \\
    $a_k(n)$, $n_c$ & learning curve of concept $k$, rows of its stored context on visit $c$ \\
    $\Delta_{k,c}(L)$ & gain of recall over relearning \\
    $L_{t,k}$, $\eta_2$, $\gamma$, $v_{t,k}$ & discounted loss, step ($\tfrac12$), discount (0.9) and weight of the slow level \\
    \bottomrule[1pt]
  \end{tabular}
\end{table}

This appendix states formally the results that \cref{sec:method} uses to justify the design of \method{}. The change
test \eqref{eq:cp}, the fast and slow levels \eqref{eq:fast}, \eqref{eq:outer-loss}--\eqref{eq:outer} and
\cref{prop:recall,prop:outer} are stated there. Proofs are in \cref{app:proofs}.

The risk of a context and the learning curve of a concept let stored contexts be compared with contexts relearned from new rows. For a concept $P$ over $(\vx,y)$ and a loss $\ell$, the \emph{risk} of a context $\cC$ is
\[
  \epsilon_P(\cC)=\E_{(\vx,y)\sim P}\,\ell\bigl(f_\Theta(\cdot\mid\vx;\cC),y\bigr),
\]
and the \emph{learning curve} of $P$ is $R_P(n)=\E_{\cC\sim P^n}\,\epsilon_P(\cC)$. We write $p_\vx=P(\cdot\mid\vx)$
for the conditional, $P_X$ for the marginal. For a predictor $h$, $\epsilon_P(h)$ is its risk on $P$, which makes
$\epsilon_P(\cC)$ the risk of the predictor $h=f_\Theta(\cdot\mid\cdot\,;\cC)$. \Cref{def:excess} uses the population
predictors of \cref{ass:ppd}, and \eqref{eq:dt-hat} estimates it from finite contexts.

\subsection{Which Contexts Share a Concept}
\label{sec:th-disc}

We compare two concepts through the TFM conditioned on their contexts. Its population behaviour follows from the principle behind prior-fitted networks, which approximate the posterior predictive distribution \citep{muller2022pfn,nagler2023statistical}.

\begin{assumption}[Posterior-predictive limit]\label{ass:ppd}
With a context from $P$ alone, the TFM predicts the conditional $p_\vx$. With contexts of equal size from $P$ and
$Q$, it predicts the conditional of the mixture $U=\tfrac12(P+Q)$,
\[
  u_\vx=\alpha(\vx)\,p_\vx+\bigl(1-\alpha(\vx)\bigr)\,q_\vx,\qquad \alpha=\frac{P_X}{P_X+Q_X}.
\]
We denote the two predictors by $h_P$ and $h_U$.
\end{assumption}

\begin{definition}[Recall regret]\label{def:excess}
For two concepts $P$ and $Q$, the \emph{recall regret} of $P$ on $Q$ is the risk that the context of $P$ incurs on $Q$
beyond that of the context of $Q$ itself,
\begin{equation}\label{eq:transfer}
  \Div_T(Q\,|\,P)=\epsilon_Q(h_P)-\epsilon_Q(h_Q).
\end{equation}
\end{definition}

\subsection{Which Rows Form a Concept}
\label{sec:th-seg}

\Cref{prop:recall} in \cref{sec:m-recall} compares relearning with the stored expert \emph{of concept $k$}. A stored context is built
from segments of the stream, and a segment that straddles a concept boundary holds rows of two concepts. The next
results say how much mixing identification tolerates, how reliably a change detector declares changes, and how
much mixing remains when its declarations lie close to the boundaries.

\begin{proposition}[Majority purity]\label{prop:purity}
Let \cref{ass:ppd} extend to unequal shares: a context whose rows come from $P$ in proportion $\pi$ and from $Q$ in
proportion $1-\pi$ predicts the conditional of $\pi P+(1-\pi)Q$. Let $P_X=Q_X$, let both concepts be deterministic,
$y=g_P(\vx)$ and $y=g_Q(\vx)$, and let $\ell$ be the 0-1 loss. Then the recall regret of this context on $P$ is
\[
  \Div_T\bigl(P\,|\,\pi P+(1-\pi)Q\bigr)=
  \begin{cases}
    0, & \pi>\tfrac12,\\
    \rho, & \pi<\tfrac12,
  \end{cases}
\]
where $\rho=\Pr_\vx[g_P(\vx)\ne g_Q(\vx)]$. At $\pi=\tfrac12$ it depends on the tie rule.
\end{proposition}

The false-alarm bound of \cref{prop:cp} concerns the test \eqref{eq:cp} of \cref{sec:m-lib} under the following assumption on the window's errors.

\begin{assumption}[Window error]\label{ass:err}
Given all earlier rows and errors, the window errs on the rows of a batch independently, each with the conditional
error rate of that batch, and almost surely between two changes this rate does not increase from batch to batch.
\end{assumption}

\begin{proposition}[Detection]\label{prop:cp}
Under \cref{ass:err}, if no change occurs in the batches $t-H_{\mathrm{cp}}-R_{\mathrm{cp}},\dots,t-1$,
  the probability that \eqref{eq:cp} declares one at $t$ is at most $(R_{\mathrm{cp}}-1)H_{\mathrm{cp}}\delta_{\mathrm{cp}}$.
\end{proposition}

\begin{corollary}[Segment purity]\label{cor:cp}
If every declared change lies within $\nu$ rows of a concept boundary and no boundary is missed, then a closed
segment $S$ holds at most $2\nu$ rows of other concepts, and it is majority-pure when $4\nu<|S|$.
\end{corollary}

\subsection{The Fast Level and When Recall Pays}
\label{sec:th-fast}

The fast level of \method{} restarts every batch with the weights \eqref{eq:fast}. Its guarantee comes from a margin.

\begin{assumption}[Margin]\label{ass:margin}
One expert $k^\star$ leads the last labelled batch by a margin $g>0$: $\ell_{k^\star}\le\ell_k-g$ for all
$k\ne k^\star$.
\end{assumption}

\begin{lemma}[Margin]\label{lem:margin}
Under \cref{ass:margin}, let $\kappa=\ln\bigl(1+(K-1)e^{-\eta g}\bigr)$. The mixture with the weights
\eqref{eq:fast} satisfies:
\begin{enumerate}
  \item[(i)] the leader holds almost all weight, $w_{k^\star}\ge e^{-\kappa}$.
  \item[(ii)] on every row, the log-loss of the mixture is at most that of expert $k^\star$ plus $\kappa$.
  \item[(iii)] the mixture predicts the class of expert $k^\star$ on every row where the top two probabilities of
  that expert differ by more than $e^{\kappa}-1$.
\end{enumerate}
\end{lemma}

\begin{proposition}[One batch of identification]\label{prop:block}
Consider a block of $J$ batches of one concept, with log-losses clipped at $c$, and let \cref{ass:margin} hold for
every labelled batch of the block. Then the regret of the fast level per row, summed over the batches of the block,
satisfies
\[
  \Regret_{\mathrm{ME}}\;\le\;c+(J-1)(K-1)\,e^{-\eta g} .
\]
\end{proposition}

\begin{corollary}[When memory pays]\label{cor:flat}
On the first visit $n_1=0$ and the area term of \eqref{eq:area} is zero. On later visits it is zero whenever the
learning curve is flat, $a_k(B)=a_k(M)$, and in general
\[
  \Delta_{k,c}(L)\;\le\;\frac{M}{L}\,\bigl(a_k(M)-a_k(B)\bigr)+\frac1J .
\]
\end{corollary}

\subsection{The Slow Level: A Safeguard at Any Time}
\label{sec:th-slow}

For $\gamma=1$ the slow level is the exponentially weighted average forecaster with an exp-concave loss
\citep{cesabianchi2006prediction}, and $\gamma<1$ forgets losses older than about $1/(1-\gamma)$ batches.
\Cref{prop:outer} bounds its regret. Synthetic expert predictions and labels
chosen to hurt the mixture or the current leader ($K\in\{2,5\}$, $\gamma\in\{0.5,0.9,0.99,1\}$, 400 rounds) keep the
discounted regret below $\ln K/\eta_2$ at $\eta_2=\tfrac12$ in every configuration (at most 0.67 for $K=2$). The
regret exceeds this bound at $\eta_2=2$ for $\gamma\ge0.9$ and reaches 4.4 at $\eta_2=10$, $\gamma=0.9$.

\section{Algorithm and Implementation}
\label{app:impl}

\begin{algorithm}[t]
\caption{\method{} for one batch $t\ge1$}\label{alg:mice}
\begin{algorithmic}[1]
\REQUIRE batch $X_t$, labels $y_{t-1}$, windows $\mathcal W$, pool $\mathcal P$, predictions $\hat p_{t-1,k}$,
  $\hat p^{\mathrm{fast}}_{t-1}$, $\hat p^{\mathrm{def}}_{t-1}$ cached for batch $t-1$ before $y_{t-1}$ arrived
\ENSURE prediction $\hat p_t$ for $X_t$
\STATE \textit{// 1. Score the predictions issued for batch $t-1$}
\IF{$t=1$}
  \STATE $\ell_k\leftarrow0$ for every expert $k$, and $L_{0,\mathrm{def}},L_{0,\mathrm{fast}}\leftarrow0$
\ELSE
  \STATE $\ell_k\leftarrow$ mean log-loss of $\hat p_{t-1,k}$ on $y_{t-1}$, for every expert $k$
  \STATE $L_{t-1,j}\leftarrow\gamma L_{t-2,j}+{}$mean half Brier score of $\hat p^{j}_{t-1}$, for
    $j\in\{\mathrm{def},\mathrm{fast}\}$
\ENDIF
\STATE \textit{// 2. Update the contexts}
\STATE add batch $t-1$ to every window in $\mathcal W$
\IF{\eqref{eq:cp} declares a change on the errors of the 300-row window, \textbf{or} the open segment has 500 rows
  and three more batches are labelled}
  \STATE close the open segment $S$ \hfill\textit{// kept only if $|S|\ge300$}
  \STATE $k^\star\leftarrow\argmin_{k\in\mathcal P}\widehat{\Div}_T(S\,|\,\cC_k)$
  \IF{$\widehat{\Div}_T(S\,|\,\cC_{k^\star})\le\delta$}
    \STATE merge $S$ into $\cC_{k^\star}$ \hfill\textit{// keeps its identity and losses}
  \ELSE
    \STATE add $S$ to $\mathcal P$ as a new expert $k'$ with $\ell_{k'}\leftarrow\operatorname{median}_k\ell_k$
  \ENDIF
\ENDIF
\STATE \textit{// 3. Predict}
\STATE $\hat p_{t,k}\leftarrow f_\Theta(\cdot\mid X_t;\cC_k)$ for every $k\in\mathcal W\cup\mathcal P$
\STATE $\hat p^{\mathrm{fast}}_t\leftarrow\sum_k w_k\,\hat p_{t,k}$, \ $w_k\propto e^{-\eta\ell_k}$
  \hfill\textit{// fast level}
\STATE $\hat p^{\mathrm{def}}_t\leftarrow\hat p_{t,M}$ \hfill\textit{// FIFO}
\STATE $\hat p_t\leftarrow v_{\mathrm{def}}\,\hat p^{\mathrm{def}}_t+v_{\mathrm{fast}}\,\hat p^{\mathrm{fast}}_t$, \
  $v_j\propto e^{-\eta_2\gamma L_{t-1,j}}$ \hfill\textit{// slow level}
\STATE cache $\hat p_{t,k}$, $\hat p^{\mathrm{fast}}_t$ and $\hat p^{\mathrm{def}}_t$, and \textbf{return} $\hat p_t$
\end{algorithmic}
\end{algorithm}

\Cref{alg:mice} gives one batch of \method{}. Before $t=1$, batch 0 is observed with its labels and predicted by no
one. It only fills the windows and counts towards the first segment. At $t=1$ nothing has been issued yet, which
makes the fast weights uniform, and $L_{0,k}=0$ makes the slow weights uniform. Scoring of cached predictions begins
at $t=2$, where $L_{1,k}=\gamma L_{0,k}+\ell_{1,k}$. A stored expert created at a segment close has issued no
prediction for that batch and starts with the median of the other experts' losses. A merged expert keeps its identity
and its losses.

\begin{figure*}[t]
  \centering
  \includegraphics[width=\textwidth]{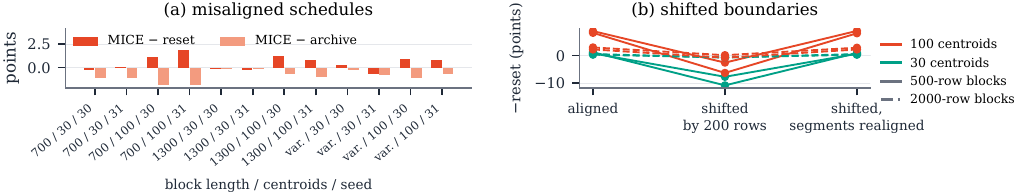}
  \caption{Boundaries inside fixed segments lower the gain of stored contexts over reset, and realigning the segments restores it.
  (a) Misaligned block schedules. (b) Boundaries shifted by 200 rows, and segments realigned.}
  \Description{Per-stream recall gains for aligned, shifted and realigned segment boundaries.}
  \label{fig:phase}
\end{figure*}

\section{Registration and Development}
\label{app:prereg}

Hypotheses, data and decision rules were written down before each test, every test stream was run once, and all
constants were set on synthetic development streams. All reported numbers are replays of cached expert predictions
with weights that use earlier labels only, which makes them equal to those of an online run.

The real streams were opened in five stages, of which the first three developed the slow level and the fourth was
registered for fixed segmentation. The main text calls the third held-out set the held-out segments and the fourth
held-out set the confirmation segments. The second held-out set consisted of all streams of the public collection
with at least 5000 rows, at most ten classes and no earlier use. No method had loaded the nine independent
confirmation segments of Bitcoin Heist and CIC-IDS2017 before their test. Detected segmentation was registered and
tested separately. \method{} w/o slow level met non-inferiority to reset on the first two held-out sets but was the
most accurate method on only one of six streams of the second, which motivated the slow level. From then on these
streams count as development data. On the third held-out set the registered slow step was 10. The step
$\eta_2=\tfrac12$ of \cref{prop:outer} was added before any two-level result was computed.

Each registered test below gives its criterion and then its outcome.
\begin{itemize}[leftmargin=*,topsep=2pt,itemsep=1pt]
  \item \emph{Grid, seeds 10 and 11} (\method{} w/o slow level): gain over reset on all eight streams with 30 or 100 centroids,
  Spearman correlation of predicted and observed gain at least 0.8. Held (0.94).
  \item \emph{Grid, seeds 20 and 21}: within 0.5 points of \method{} w/o slow level and above reset on the same eight streams.
  Held.
  \item \emph{Oracle test} of \cref{fig:rq1}b: measured area within 1.0 point of its prediction on at least five of six
  cells, Spearman at least 0.8. Held.
  \item \emph{Confirmation segments}, the fourth held-out set (six segments, fixed segmentation): never more than 0.5
  points below the better of FIFO and reset, more accurate than all in-context baselines on four of six with the
  highest mean, never more than 0.3 points below FIFO. Held. A second backbone seed held on these six segments. On
  four segments of the third set it gave two of four where three were registered. With a second backbone seed, \method{} is the
most accurate policy on two of four held-out segments. With TabICL as the backbone and no constant changed, \method{}
is the most accurate policy on all six confirmation segments.
  \item \emph{Boundary tests} (fixed segmentation): above reset on 10 of 12 misaligned streams and on 7 of 8 shifted
  streams, each with a mean gain of one point. Both failed. The realigned test added afterwards held.
  \item \emph{Detected segmentation} (72 test streams, criteria C1--C5): gain over reset of at least one point in each
  length regime with a positive bootstrap bound (C1), recovery on the shifted grid (C2), no loss on the aligned grid
  (C3) and on the real streams (C4), and fewer calls than an archive (C5). All held. Purity: 88\% of the closes held at
  least 95\% rows of one concept, short of the registered 90\%.
  \item \emph{Pool controls}: \method{} above the archive by 0.5 points on the aligned grid failed ($+0.30$). The
  real-stream criterion held.
  \item \emph{Confirmation of detected segmentation} (nine independent confirmation segments): the one-sided 95\%
  bootstrap lower bound on the mean difference to fixed segmentation is above $-0.2$ points (H1), the corresponding
  per-segment bound against FIFO is above $-0.3$ points (H2), observed accuracy is never more than 0.5 points below
  the better of FIFO and reset (H3), and \method{} is above the pool-free replay on at least six of nine (H5). All
  four held. Being the most accurate per source and on at least six of nine against the four in-context baselines
  (H4) failed ($-0.06$ and $-0.01$ points, four of nine).
\end{itemize}

A registered reporting rule required the fixed-segmentation capacity comparison to name in the abstract any stream group with a mean gain of at most $+0.1$ points over a time-budgeted pool. The shifted-grid snapshot
comparison meets this condition and is reported there. Capacity was chosen from prefix timings. The four-expert
configuration is supplementary and exceeds the time budget. The analogous detected-segmentation reporting rule also
required naming a qualifying group in the conclusion, and no group met its trigger. Pool controls include both
segmentation variants, identified with each result.  With four experts, which take more time per batch, both pools are level with \method{} on the shifted grid, as \cref{fig:pools}a shows.  

The registration deviated in its seed use and in the stored-row form of the recall predictions, which \cref{sec:rq1} reports. Seeds 40 and 41, registered as unused for the detected-segmentation test, had been used by the oracle test. We kept the verdicts and repeated the test on unused seeds 50 to 53, where C1 held. 

The rows from ARF to Learn++.NSE in \cref{tab:main} come from a rerun of the trained learners with stored probabilities, whose mean accuracy is within 0.2 points of the first run.

\begin{figure}[!t]
  \centering
  \includegraphics[width=\columnwidth]{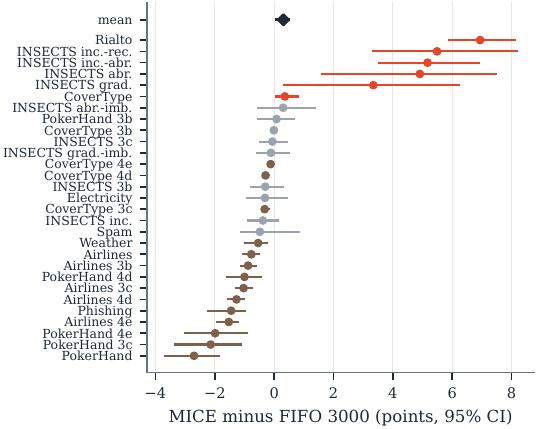}
  \caption{\method{} (fixed segments) minus a 3000-row FIFO window on the 29 real streams, with 95\% bootstrap
  intervals.}
  \Description{Per-stream accuracy differences between fixed-segmentation MICE and a 3000-row FIFO window, with bootstrap intervals.}
  \label{fig:fifo3000}
\end{figure}

\section{Context Stress Test: Protocols}
\label{app:diag}

The three parts of the context stress test of \cref{sec:diag} reuse runs of \cref{sec:exp}, and no run was added for them. 

The revisit test measures how a policy recovers when a concept returns after FIFO has discarded its rows. It uses the Sine, STAGGER and Agrawal generator streams of \cref{sec:exp} from River \citep{montiel2021river}, with blocks of 2000 rows. Each generator supplies five streams from seeds 10 to 14, with three cycles over three or four concepts. Accuracy is averaged over the five batches after a concept recurs, that is, after the
start of a block whose concept appeared in an earlier block. The policies are FIFO, reset and an oracle context of past rows of the recurring concept, which knows the concept and its boundaries.

The splice test measures what mixing concepts costs by moving the boundaries inside fixed segments. It uses the eight grid streams of \cref{app:grid} with 30 or 100 centroids and blocks of 500 or 2000 rows. These streams use seeds 20 and 21, with shifted and aligned copies compared on matched batches. Each shifted copy prepends 200 rows of the first concept, drawn from the same
centroids with a separate random generator. Every boundary therefore lies 200 rows into a 500-row segment and stays on a batch boundary. Stored contexts are formed by the fixed segmentation of \cref{sec:method}. The realigned run starts the segments at row 200, the known phase of the boundaries. The quantity is \method{} with
fixed segmentation minus reset, over all batches and over the first five batches of each recurring block. \Cref{fig:phase} gives every stream.

The curve test estimates relearning speed from learning curves measured before the grid test on data seeds 10 and 11. Each centroid count and concept supplies one block of 4000 rows, of which rows 1--3000 form the context pool and rows 3001--4000
the 1000 test rows. For $n\in\{50,100,200,500,1000,2000\}$, three contexts are drawn without replacement from the
pool, and the TFM predicts the test rows. The curve is the mean accuracy over 2 seeds, 3 concepts and 3 repetitions, as \cref{app:grid} details.

\section{Streams and Baselines}
\label{app:details}

The grid and generator streams are produced by \path{code/grid_streams.py} and \path{code/streams.py}, and long real streams are cut into segments as follows.
Real streams with more than 100{,}000 rows are cut into segments of 100{,}000 (CoverType and one INSECTS stream:
150{,}000 for the first segment).

LTM follows the buffer allocation of \citet{lourenco2026streams}, which gives its short-term buffer a share of 0.75.

\section{Proofs}
\label{app:proofs}

\subsection{Majority Purity, Detection and Segment Purity}

\emph{\Cref{prop:purity}.} With $P_X=Q_X$ the conditional of $\pi P+(1-\pi)Q$ is
$\pi\,\mathbf 1[y=g_P(\vx)]+(1-\pi)\,\mathbf 1[y=g_Q(\vx)]$. Where $g_P(\vx)=g_Q(\vx)$, its mode is $g_P(\vx)$. Where
they differ, the mode is $g_P(\vx)$ if $\pi>\tfrac12$ and $g_Q(\vx)$ if $\pi<\tfrac12$. The 0-1 risk on $P$ is therefore
$0$ or $\rho$, and $\epsilon_P(h_P)=0$, which gives the regret.

\emph{\Cref{prop:cp}.} Fix a candidate $b$ and a number $m\in\{2,\dots,R_{\mathrm{cp}}\}$ of compared batches before
it, and write the statistic of \eqref{eq:cp} without the margin as
$Z_m=\frac1{n_2}\sum_{\mathrm{post}}\mathbf 1[\text{error}]-\frac1{n_1}\sum_{\mathrm{pre}}\mathbf 1[\text{error}]$ over
the rows of the batches $b-m,\dots,b-1$ (pre, $n_1=mB$) and $b,\dots,t-1$ (post), and $\bar Z_m$ for the same sum of
conditional error rates. For fixed $b$ and $m$ these rows and the weights $1/n_1$ and $1/n_2$ do not depend on the
data. By \cref{ass:err} the error indicators minus their conditional rates, taken in row order, form a martingale
difference sequence whose terms lie in intervals of length $1/n_2$ (post) and $1/n_1$ (pre). By the
Azuma--Hoeffding inequality, $\Pr[Z_m-\bar Z_m\ge\lambda]\le\exp\bigl(-2\lambda^2/(1/n_1+1/n_2)\bigr)$, and
at $\lambda=\tau(n_1,n_2)$ this bound equals $\delta_{\mathrm{cp}}$. The reference that \eqref{eq:cp} uses starts
after the previous declared change. This start depends on earlier errors, possibly on those of the compared rows, but
it is always one of these $R_{\mathrm{cp}}-1$ fixed references, and a union bound over $m$ covers this selection.
Without a change in the batches $t-H_{\mathrm{cp}}-R_{\mathrm{cp}},\dots,t-1$, \cref{ass:err} gives
$\bar Z_m\le0$ for every $b$ and $m$. Each excess is therefore positive with probability at most
$\delta_{\mathrm{cp}}$, and a union bound over at most $H_{\mathrm{cp}}$ candidates and $R_{\mathrm{cp}}-1$
references gives the claim.

\emph{\Cref{cor:cp}.} Each end of a closed segment is a declared change, a periodic close inside one concept, or
the start of the stream. Without missed boundaries, the rows of other concepts lie between an end and the true
boundary nearest to it. These are at most $\nu$ rows at each end and hence at most $2\nu$ in all. The majority
concept therefore holds more than half of the rows when $2\nu<|S|/2$.

\subsection{Identification under a Margin}

\begin{proof}[Proof of \cref{lem:margin}]
\emph{(i)} By \cref{ass:margin},
\begin{align*}
  w_{k^\star}&=\Bigl(1+\sum_{k\ne k^\star}e^{-\eta(\ell_k-\ell_{k^\star})}\Bigr)^{-1}\\
  &\ge\bigl(1+(K-1)e^{-\eta g}\bigr)^{-1}=e^{-\kappa}.
\end{align*}
\emph{(ii)} For a row with label $y$ and expert predictions $p_k$,
\[
  -\ln\sum_kw_kp_k(y)\;\le\;-\ln\bigl(w_{k^\star}p_{k^\star}(y)\bigr)\;\le\;-\ln p_{k^\star}(y)+\kappa .
\]
\emph{(iii)} For two classes $y_1,y_2$,
\begin{align*}
  \sum_kw_k\bigl(p_k(y_1)-p_k(y_2)\bigr)\;&\ge\;w_{k^\star}\bigl(p_{k^\star}(y_1)-p_{k^\star}(y_2)\bigr)\\
  &\quad-(1-w_{k^\star}),
\end{align*}
which is positive when the gap of expert $k^\star$ exceeds $(1-w_{k^\star})/w_{k^\star}$, a quantity at most $e^{\kappa}-1$.
\end{proof}

\begin{proof}[Proof of \cref{prop:block}]
The first batch of the block is predicted with weights from a batch of the previous concept, and its per-row
regret is at most $c$. Every later batch is predicted with weights from a labelled batch of the block on which
$k^\star$ leads by $g$. \Cref{lem:margin}(ii) therefore bounds its per-row regret by
$\ln\bigl(1+(K-1)e^{-\eta g}\bigr)\le(K-1)e^{-\eta g}$. Summing over the $J-1$ later batches gives the claim.
\end{proof}

\subsection{Recall versus Relearn}

\begin{proof}[Proof of \cref{prop:recall}]
The first batch of the block is predicted by both learners before any label of the block has arrived. Its
accuracy difference, divided by $J$, is $\zeta_{k,c}$ and is at most $1/J$ in absolute value because
accuracies lie in $[0,1]$. Before batch $j+1$ ($j=1,\dots,J-1$), the relearning context holds $\min(jB,M)$ rows of
the block. Its expected accuracy is therefore $a_k(\min(jB,M))$, whereas the stored expert has expected accuracy
$a_k(n_c)$. By \cref{lem:margin}(iii), the mixture predicts as the better of the two. Its expected accuracy is thus
the larger of the two values, and its difference to relearning is the positive part in \eqref{eq:area}. Averaging over the $J$ batches gives the area term, and adding
the first batch gives $\zeta_{k,c}$.
\end{proof}

\begin{proof}[Proof of \cref{cor:flat}]
On the first visit $n_1=0$, and every positive part vanishes because $a_k$ is non-decreasing and hence
$a_k(0)\le a_k(\min(jB,M))$. Since $a_k$ is non-decreasing and $n_c\le M$, every term of the area term is at most
$a_k(M)-a_k(B)$, which is zero for a flat curve, and the terms with $jB\ge M$ vanish. At most $M/B$ of the $J-1$
terms are non-zero, namely those with $jB<n_c\le M$. Combining this count with $|\zeta_{k,c}|\le1/J$ gives the
bound.
\end{proof}

\subsection{Any-Time Safeguard}

\begin{proof}[Proof of \cref{prop:outer}]
\emph{Exp-concavity.} The half Brier score $\ell(p)=\tfrac12\lVert p-e_y\rVert^2$ has gradient $p-e_y$ and Hessian
$I$, and $\lVert p-e_y\rVert^2\le2$ on the simplex. Hence
$\nabla^2\ell\succeq\eta_2\,\nabla\ell\nabla\ell^{\top}$ for $\eta_2\le\tfrac12$, that is, $\ell$ is
$\eta_2$-exp-concave. The mean over the rows of a batch is $\eta_2$-exp-concave as well: its Hessian is the mean
Hessian, and $\tfrac1n\sum_ig_ig_i^\top\succeq\bar g\bar g^\top$ for the row gradients $g_i$ with mean $\bar g$.
Consequently,
\begin{equation}\label{eq:expconcave}
  e^{-\eta_2\ell_t(p_t)}\;\ge\;\sum_kv_{t,k}\,e^{-\eta_2\ell_{t,k}} .
\end{equation}
\emph{Potential.} Let $W_t=\sum_ke^{-\eta_2L_{t,k}}$ and note that $W_0=K$. By \eqref{eq:outer-loss} and
\eqref{eq:outer-weight},
\[
  W_t=\Bigl(\sum_ke^{-\eta_2\gamma L_{t-1,k}}\Bigr)\sum_kv_{t,k}\,e^{-\eta_2\ell_{t,k}} .
\]
By the power-mean inequality, $\sum_kx_k^{\gamma}\le K^{1-\gamma}\bigl(\sum_kx_k\bigr)^{\gamma}$ for $x_k\ge0$ and
$\gamma\in(0,1]$. Applying it with $x_k=e^{-\eta_2L_{t-1,k}}$ and combining with \eqref{eq:expconcave} gives
\[
  \ln W_t\;\le\;(1-\gamma)\ln K+\gamma\ln W_{t-1}-\eta_2\,\ell_t(p_t).
\]
\emph{Unrolling.} Since $(1-\gamma)\sum_{t\le T}\gamma^{T-t}+\gamma^{T}=1$,
\[
  \ln W_T\;\le\;\ln K-\eta_2\sum_{t=1}^{T}\gamma^{T-t}\ell_t(p_t).
\]
On the other hand, $\ln W_T\ge-\eta_2L_{T,k}=-\eta_2\sum_{t=1}^{T}\gamma^{T-t}\ell_{t,k}$ for every $k$. Combining the
two inequalities gives \eqref{eq:outer-bound}.
\end{proof}

\section{Controlled Grid and Learning Curves}
\label{app:grid}

The grid generator fixes the input distribution and lets concepts differ only in their classes. A stream has $d=10$ features, four nominal classes and three concepts. For $m\in\{5,30,100\}$
centroids, $C_j\sim\mathcal U([0,1]^{10})$ i.i.d.\ and mixture weights $w_j=u_j/\sum_i u_i$ with $u_j\sim\mathcal U(0,1)$.
Each concept $k$ has its own class vector $\lambda_k(j)\sim\mathcal U\{0,1,2,3\}$, drawn i.i.d.\ over centroids and
independently across concepts. A row of concept $k$ draws $j\sim\mathrm{Categorical}(w)$,
$\vx=C_j+\mathcal N(0,0.15^2 I_{10})$ and $y=\lambda_k(j)$. Concepts appear in the order $1,2,3$ in blocks of
$L\in\{500,2000\}$ rows for three cycles. The input distribution $P_X$ is the same for all concepts because all of them share $C$, $w$ and the noise.
One NumPy generator per data seed draws $C$, $w$ and $\lambda$ first and then the rows.

Because the class vectors are drawn independently, $\lambda_k$ is not a permutation and two concepts may agree on a
centroid. On our seeds the weight of such centroids is 18\% to 48\% averaged over the three pairs of concepts and
3\% to 76\% for single pairs. With five centroids a concept can also lack some classes, with two to four present.
The generator applies no rejection or coverage constraint, and class priors follow from $w$ and $\lambda_k$.

The data seeds are allocated as follows. Seeds 0 and 1 serve development, including that of the
change detector, and seeds 10 and 11 the registered test. Seeds 20 and 21 serve confirmation, 30 and 31 the
misaligned schedules, and 40 and 41 the oracle test. Seeds 40 to 43 serve the test of detected segmentation and share seeds 40 and 41 with the
oracle test. Seeds 50 to 53 serve its confirmation. The shifted streams reuse seeds 20 and 21 and draw their prefix from a separate
generator. The streams of \cref{fig:segments} prepend rows of the first concept from a separate generator and also
use blocks of lengths drawn uniformly from 500 to 2\,500 rows.

\begin{table}[t]
  \centering
  \caption{Controlled grid: parameters.}
  \label{tab:gridparams}
  \small
  \resizebox{\ifdim\width>\columnwidth\columnwidth\else\width\fi}{!}{%
  \begin{tabular}{ll}
    \toprule[1pt]
    Features, classes, concepts & $d=10$, 4, 3 \\
    Centroids $m$ & 5, 30, 100 with $C_j\sim\mathcal U([0,1]^{10})$ \\
    Weights & $u_j\sim\mathcal U(0,1)$, normalised \\
    Class of centroid $j$ in concept $k$ & $\lambda_k(j)\sim\mathcal U\{0,\dots,3\}$, independent \\
    Noise & isotropic Gaussian, s.d.\ 0.15 \\
    Block length $L$, cycles & 500 or 2000 rows, 3 \\
    Batch size, $M$ & 100, 1000 rows \\
    Segments & 300--700 rows (detected), 500 rows (fixed) \\
    \bottomrule[1pt]
  \end{tabular}}
\end{table}

The learning curves that supply the recall predictions were measured before the grid test, for data seeds 10 and 11. For each $m$ and concept, one block of 4000 rows of that concept comes from the generator above. Rows 1--3000 form the context pool and rows 3001--4000
the 1000 test rows, disjoint. For $n\in\{50,100,200,500,1000,2000\}$, three contexts are drawn without replacement from
the pool (generator seeds 0, 1, 2). Contexts of different sizes and repetitions therefore overlap, while the same 1000
test rows are used throughout. The frozen TabPFN v2 with four estimators and random state equal to the data seed
predicts the test rows, whereas stream evaluation uses backbone random state~0. The curve $a(n)$ is the mean accuracy over 2 seeds, 3 concepts and 3 repetitions (18 measurements per $n$, not
18 independent concept draws), interpolated linearly in $\log n$.

Predicted recall gains come from inserting the averaged curve into the area term of \cref{eq:area}, which measures the accuracy that a stored expert saves over relearning. Visits 2 and 3 of a concept recur and enter with weight $\tfrac13$ each. The first visit has no stored expert and contributes zero. The registered predictions used $M$ stored rows on every recurrence, and \cref{sec:rq1} reports the form with
$n_c=\min(M,(c-1)L)$. The observed gain of a cell is the mean over its two streams, and the Spearman correlations are
over the six cells.

The curves of seeds 10 and 11 share randomness with the evaluated streams of these seeds, because their data were generated by re-initialising the generator of the same seed with one 4000-row block per concept. The shared part is the concept definitions of the evaluated seed-10/11 streams, the centroid draws of their first
block, and some bit-identical feature vectors (0 to 1887 per stream). These identical rows arise in runs where the
two draw sequences of the same seed realign at a constant offset. The seed-20/21 streams, by contrast, share nothing
with the curves.

To remove the shared rows, a sensitivity analysis registered after this overlap was found remeasured the curves on
rows drawn, given the same concept definitions, from one independent NumPy generator per seed, centroid count and
concept. Each generator is seeded with $(20\,000+\text{seed},m,k)$, and no redrawn row is in common with any
evaluated stream. These counts and checks are produced by \texttt{code/grid\_overlap\_check.py}. The redraw changes
the curves by at most 0.015 at any $n$ and the predictions by at most 0.6 points. For 100 centroids and $L=500$, the
prediction is 9.9 instead of the registered 10.4, and 8.9 instead of 9.2 with $n_c$. The rank correlations are
unchanged at 0.94 on seeds 10 and 11 and 0.83 on seeds 20 and 21.

\section{Per-Stream Results and Resources}
\label{app:perstream}

\Cref{fig:realdiff} gives the accuracy of every policy minus FIFO on each real stream.

Per-stream accuracy comparisons show how the effects of pool organisation and a larger FIFO window vary across streams. \Cref{fig:poolctl} shows the pool controls of
\cref{sec:ablation} per stream, and \cref{fig:fifo3000} the comparison with a 3000-row window.

\begin{figure*}[t]
  \centering
  \includegraphics[width=\textwidth]{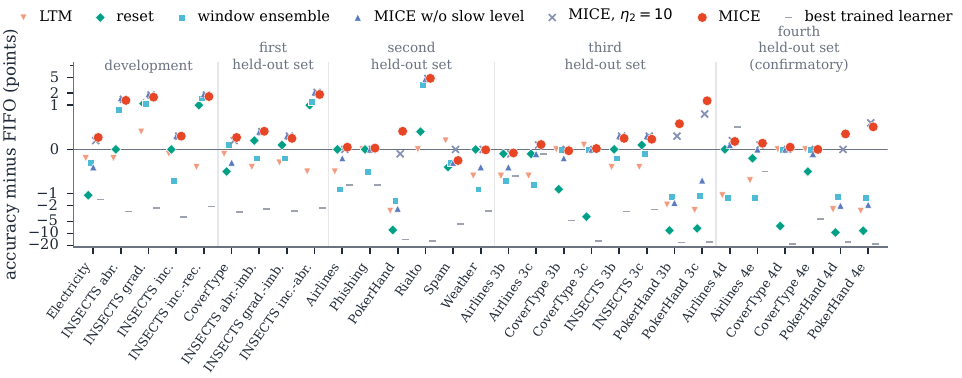}
  \caption{Accuracy minus FIFO on each real stream, grouped by evaluation stage (fixed segments).}
  \Description{Per-stream accuracy differences of context policies relative to FIFO, grouped by evaluation stage.}
  \label{fig:realdiff}
\end{figure*}

\begin{figure*}[t]
  \centering
  \includegraphics[width=\textwidth]{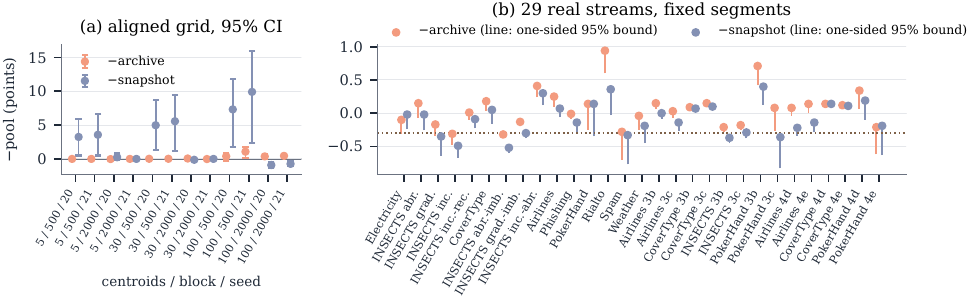}
  \caption{\method{} minus the archive and the snapshot pool of size 12, fixed segments. (a) Aligned grid. (b) The 29
  real streams, with a dotted line at $-0.3$ points.}
  \Description{Per-stream accuracy differences between MICE and unmerged archive or snapshot context pools.}
  \label{fig:poolctl}
\end{figure*}

\begin{figure*}[t]
  \centering
  \includegraphics[width=\textwidth]{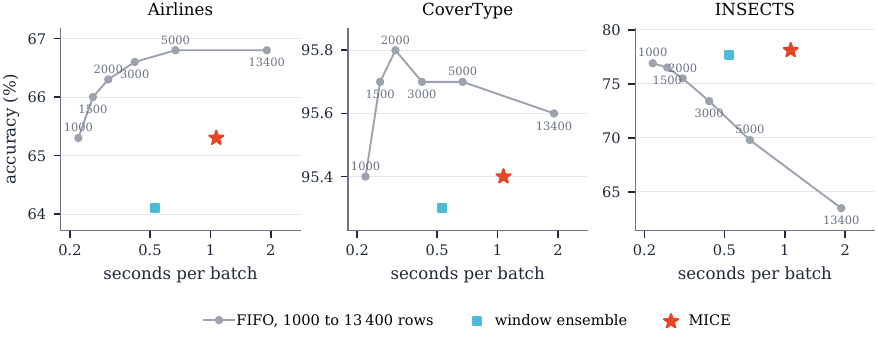}
  \caption{Accuracy against time per batch on three real prefixes for FIFO windows of 1000 to 13\,400 rows, the window
  ensemble and \method{} (fixed segments). Peak GPU memory: 0.8 to 26.4\,GB for FIFO, 1.0\,GB for \method{}.}
  \Description{Accuracy versus time per batch for different FIFO window sizes, the window ensemble and MICE on three real-stream prefixes.}
  \label{fig:cost}
\end{figure*}

\end{document}